\documentclass[10pt,journal,compsoc]{IEEEtran}

\ifCLASSOPTIONcompsoc
  \usepackage[nocompress]{cite}
\else
  \usepackage{cite}
\fi

\ifCLASSINFOpdf
\else
\fi

\usepackage[dvipsnames,svgnames,x11names]{xcolor}
\usepackage{amsmath,amsfonts,amsthm}
\usepackage{algorithm}
\usepackage{algorithmic}

\usepackage{array}
\usepackage[caption=false,font=normalsize,labelfont=sf,textfont=sf]{subfig}
\usepackage{textcomp}
\usepackage{stfloats}
\usepackage{url}
\usepackage{verbatim}
\usepackage{booktabs}
\usepackage[switch]{lineno}
\usepackage{etoolbox}
\makeatletter
\patchcmd{\@makecaption}
  {\scshape}
  {}
  {}
  {}
\makeatletter
\patchcmd{\@makecaption}
  {\\}
  {.\ }
  {}
  {}
\makeatother

\usepackage{stfloats}
\usepackage{multirow}
\usepackage{amssymb}
\usepackage[mathscr]{eucal}
\usepackage{graphicx}
\usepackage{svg}
\usepackage{xcolor}
\usepackage{colortbl}
\definecolor{increment}{rgb}{0.4627450980392157,0.5333333333333333,0.3568627450980392}
\definecolor{decrement}{rgb}{0.6078431372549019,0.2235294117647059,0.13333333333333333}
\usepackage{dsfont}
\usepackage[pagebackref,breaklinks,colorlinks]{hyperref}
\usepackage{pifont}
\usepackage{svg}
\usepackage{enumerate}
\usepackage{enumitem}

\newtheorem{prop}{Proposition}
\newtheorem{assumption}{Assumption}
\newtheorem{corollary}{Corollary}

\begin{document}

\title{A Unified Gradient-based Framework for Task-agnostic Continual Learning-Unlearning}

\author{Zhehao Huang, Xinwen Cheng, Zhang Jie, JingHao Zheng, Haoran Wang, Zhengbao He, Tao Li \\ Xiaolin Huang, ~\IEEEmembership{Senior Member,~IEEE}
\thanks{Z. Huang, X. Cheng, Z. Jie, J. Zheng, H. Wang, Z. He, T. Li and X. Huang are with Institute of Image Processing and Pattern Recognition, Shanghai Jiao Tong University, 200240 Shanghai, P.R. China. Email: $\{$kinght\_h, xinwencheng, zhangjie20000411, zjh20030406, haoran\_whynot, lstefanie, li.tao, xiaolinhuang$\}$@sjtu.edu.cn
}
}

\markboth{submitted to IEEE Transactions on Pattern Analysis and Machine Intelligence}%
{Huang \MakeLowercase{\textit{et al.}}: A Unified Gradient-based Framework for Task-agnostic Continual Learning-Unlearning}


\IEEEtitleabstractindextext{%
\begin{abstract}
Recent advancements in deep models have highlighted the need for intelligent systems that combine continual learning (CL) for knowledge acquisition with machine unlearning (MU) for data removal, forming the Continual Learning-Unlearning (CLU) paradigm.
While existing work treats CL and MU as separate processes, we reveal their intrinsic connection through a unified optimization framework based on Kullback-Leibler divergence minimization. This framework decomposes gradient updates for approximate CLU into four components: learning new knowledge, unlearning targeted data, preserving existing knowledge, and modulation via weight saliency. A critical challenge lies in balancing knowledge update and retention during sequential learning-unlearning cycles. To resolve this stability-plasticity dilemma, we introduce a remain-preserved manifold constraint to induce a remaining Hessian compensation for CLU iterations. A fast-slow weight adaptation mechanism is designed to efficiently approximate the second-order optimization direction, combined with adaptive weighting coefficients and a balanced weight saliency mask, proposing a unified implementation framework for gradient-based CLU. Furthermore, we pioneer task-agnostic CLU scenarios that support fine-grained unlearning at the cross-task category and random sample levels beyond the traditional task-aware setups. Experiments demonstrate that the proposed UG-CLU framework effectively coordinates incremental learning, precise unlearning, and knowledge stability across multiple datasets and model architectures, providing a theoretical foundation and methodological support for dynamic, compliant intelligent systems. 
\end{abstract}

}

\maketitle

\IEEEdisplaynontitleabstractindextext

%
\IEEEpeerreviewmaketitle

\IEEEraisesectionheading{\section{Introduction}\label{sec: introduction}}
Recent advancements in large-scale deep models~\cite{Kaplan2020ScalingLF} have driven a new phase of capability transformation in artificial intelligence. 
Future intelligent systems are anticipated to follow a dual paradigm: on the one hand, developing personalized systems for long-term companionship and lifelong learning in dynamic environments; on the other hand, addressing growing demands for data privacy protection and ethical constraints. This dual necessity has spurred the emergence of a new research direction, \textbf{Continual Learning-Unlearning (CLU)}~\cite{liu2022continual,chatterjee2024unifiedframeworkcontinuallearning}. In this context, intelligent agents must not only achieve capability evolution over time through continual learning but also possess reverse learning abilities, such as removing specific data traces or revising model cognition on demand.

A CLU system can be viewed as an integration of continual learning (CL)~\cite{DeLange2019ACL,Wang2023ACS,Zhou2023DeepCL} and machine unlearning (MU)~\cite{Bourtoule2019MachineU,shaik2023exploring,xu2024machine}, inheriting the characteristics of both settings. Previous research often regards CL as an incremental process of knowledge acquisition, while MU emerges as its reverse process, aiming to reduce knowledge. We hypothesize that both processes inherently share similar dynamics. Actually, CL and MU can be described by their oracle models, which serve as their upper bounds and are typically approximated by minimizing empirical loss according to their respective objectives. This insight demonstrates that the seemingly distinct two settings can both be framed as problems of empirical loss minimization transfer. Consequently, we are able to model CL and MU within a unified optimization framework. By modeling the oracle model in the CLU system, i.e., the optimal model trained on the joint dataset after removing all unlearning data, and establishing the corresponding Kullback-Leibler (KL) divergence minimization problem, we derive the vanilla gradient descent direction in general CLU scenarios. By linking this optimization process to the mechanisms of CL and MU, we find that this direction can be decomposed into four components:  \ding{172} a weighted gradient descent direction for learning new knowledge; \ding{173} a weighted gradient ascent direction to eliminate the influence of unlearned samples; \ding{174} a gradient descent direction for retaining existing knowledge; and \ding{175} a weight saliency matrix modulating the optimization direction. This decomposition provides a novel perspective for integrating CL and MU methods, offering a unified framework for constructing general CLU systems.


A core challenge in CLU systems lies in the severe degradation of previously acquired knowledge when the model assimilates new knowledge or actively unlearns specific data. 
This phenomenon, known as \textbf{catastrophic forgetting}, has been widely studied in CL through approaches such as regulization~\cite{Kirkpatrick2016OvercomingCF,Ritter2018OnlineSL,Li2016LearningWF,Hou2019LearningAU}, replay~\cite{Chaudhry2019ContinualLW,Buzzega2020DarkEF,Rebuffi2016iCaRLIC,lopez2017gradient}, and parameter isolation~\cite{chen2024mofomomentumfilteredoptimizermitigating,hui2024hfthalffinetuninglarge}, but its solution within CLU scenarios remains underexplored.

Motivated by recent studies demonstrating that deep neural networks' parameter space and training dynamics inherently reside within low-dimensional manifold structures~\cite{Li2022LowDT,Hu2021LoRALA}, a key question arises: \textit{Can we estimate the optimization trajectory of CLU models within a low-dimensional manifold space to enable effective learning of new knowledge, targeted unlearning of specific data, and simultaneous retention of existing knowledge}?

To achieve this goal, we propose to discover the descent direction under the KL divergence on the remaining output distribution. Using such a manifold metric, the optimization direction can be amended by a second-order remaining Hessian. This geometric constraint can effectively prevent interference from learning new knowledge and the operation of targeted unlearning on the general old knowledge. To address the high computational complexity of the Hessian matrix~\cite{elsayed2022hesscale}, we design a fast-slow weight method to implicitly approximate the salient directions with Hessian modulation, adding only controllable computational overhead. Furthermore, through theoretical derivation, we develop an adaptive loss weighting strategy and a balanced weight saliency selection, forming our \textbf{U}nified \textbf{G}radient-based \textbf{C}ontinual \textbf{L}earning-\textbf{U}nlearning (\textbf{UG-CLU}) method.

\begin{figure*}[tb]
    \centering
    \includegraphics[width=\textwidth]{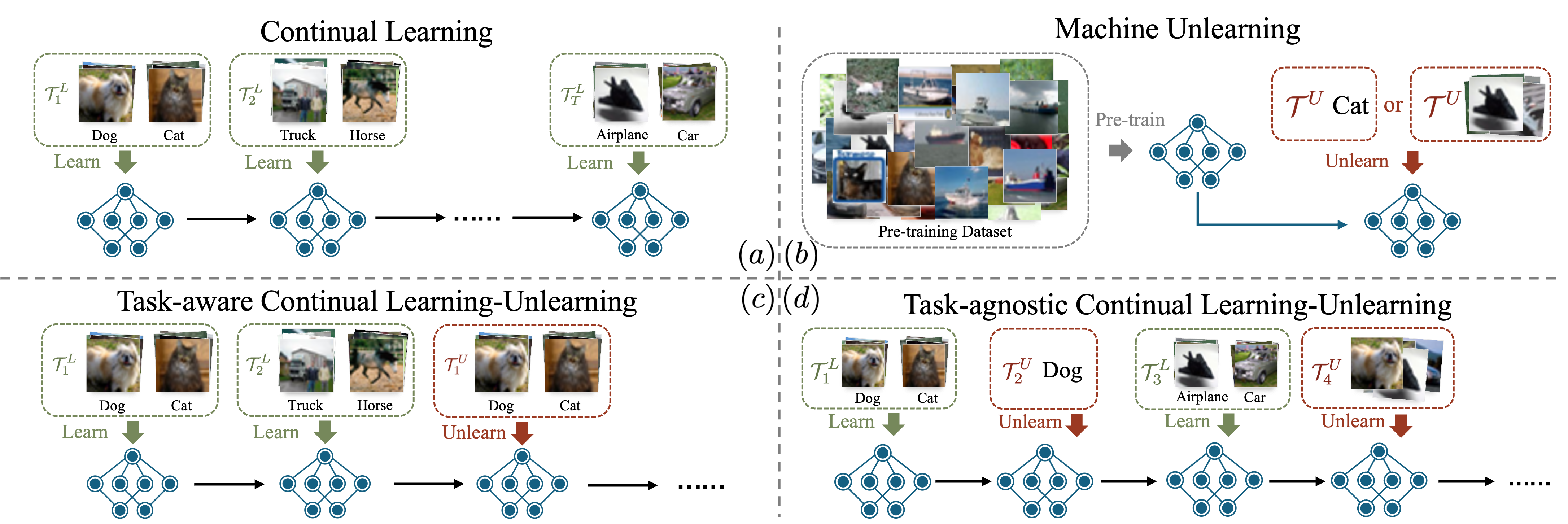}
    \caption{A visual comparison of the paradigms for CL, MU, and CLU systems, where $\mathcal T^L_t$ and $\mathcal T^U_t$ represent the learning and unlearning tasks, respectively. ($a$) Traditional \textbf{CL} systems~\cite{DeLange2019ACL,Wang2023ACS,Zhou2023DeepCL} adopt a sequential incremental learning paradigm, gradually expanding the model's capabilities through a stream of temporal task data. ($b$) Typical \textbf{MU} systems~\cite{Bourtoule2019MachineU,shaik2023exploring,xu2024machine} require removing the associated influence of specific classes or data from a well-trained model on pre-trained dataset. ($c$) Existing CLU systems~\cite{liu2022continual,chatterjee2024unifiedframeworkcontinuallearning}, while integrating CL and MU, not only need to learn from data but may also require unlearning. However, they only support task-level unlearning updates, i.e., \textbf{task-aware CLU}. ($d$) We further propose a \textbf{task-agnostic CLU} system that supports fine-grained knowledge units (classes or data samples) for targeted unlearning operations, enabling more precise control over cognitive evolution.}
    \label{fig: system overview}
\end{figure*}

Furthermore, existing research preliminarily establishes the basic paradigm of CLU~\cite{liu2022continual,chatterjee2024unifiedframeworkcontinuallearning}, but limited to task-aware CLU, restricting unlearning operations to complete tasks at a time. This constraint fails to align with practical requirements for granular knowledge management, which often demands the selective removal of specific cross-task categories or even individual data samples. To bridge this gap, we pioneer a task-agnostic CLU setup and construct two novel benchmarks for cross-task category unlearning and random sample unlearning, providing protocols for evaluating fine-grained knowledge control capabilities. The key contributions of this paper include:

\begin{itemize}
    \item[$\bullet$] We provide a novel perspective that integrates CL and MU methods to derive the optimization direction for the general CLU problem.  
    \item[$\bullet$] We obtain the optimization direction for the CLU problem within the remain-preserved manifold, preventing damage to the general retention capability.  
    \item[$\bullet$] We propose a fast-slow weight method to implicitly approximate the salient CLU update directions modulated by online Hessian and introduce the UG-CLU method by combining adaptive loss weighting and a balanced weight saliency selection.  
    \item[$\bullet$] We propose a task-agnostic realistic CLU setting, design two benchmarks for unlearning inter-task categories and random samples, and validate the effectiveness of our method across multiple datasets and architectures.
\end{itemize}

\textbf{Comparison with our conference work}. This paper builds upon the previous NeurIPS 2024 conference version~\cite{Huang2024UnifiedGM} and incorporates the following significant enhancements: \ding{172} We extend the conference version's focus on the MU scenario to the more complex CLU scenario. \ding{173} We re-model the approximate MU analysis based on the steepest descent in the conference version within the approximate CLU scenario. \ding{174} We discover that gradient-based CL and MU methods inherently share the same theoretical framework and can be modularly designed using our proposed UG-CLU framework. \ding{175} We expand and adapt the implementation method from the conference version to the CLU scenario. \ding{176} We introduce a new task-agnostic CLU setting for the CLU scenario and design two benchmarks for cross-task category and interclass confusion unlearning. Extensive experiments in the CLU scenario validate the effectiveness of the UG-CLU framework.

\section{Related Work}

\subsection{Continual Learning}
CL~\cite{DeLange2019ACL,Wang2023ACS,Zhou2023DeepCL} aims to address the critical challenge of intelligent models continuously absorbing new knowledge in sequential task streams. Its core dilemma lies in balancing the model's plasticity (the ability to integrate new knowledge) and stability (the ability to retain old knowledge) to mitigate the phenomenon of catastrophic forgetting~\cite{McCloskey1989CatastrophicII, Ratcliff1990ConnectionistMO,Zenke2017ContinualLT}. Current mainstream gradient optimization methods can be systematically categorized into the following three paradigms: 
Regularization-based methods~\cite{Kirkpatrick2016OvercomingCF,Ritter2018OnlineSL,Li2016LearningWF,Hou2019LearningAU}: These maintain model stability by constraining the evolution path in parameter space or function space. Representative approaches include: Stability constraints on parameter distributions, which use the Fisher information matrix to characterize parameter importance and limit shifts in critical weights~\cite{Kirkpatrick2016OvercomingCF,Ritter2018OnlineSL}. Knowledge distillation-guided representation alignment, which retains historical task representations through output distillation losses between old and new models~\cite{Li2016LearningWF,Hou2019LearningAU}.
Data replay-based methods~\cite{Chaudhry2019ContinualLW,Buzzega2020DarkEF,Rebuffi2016iCaRLIC,lopez2017gradient}: These explicitly store historical samples or generate pseudo-samples to reconstruct an approximate complete data distribution, enabling joint optimization of old and new tasks under the empirical risk minimization framework. While effective in mitigating forgetting, these methods face dual constraints of increasing storage overhead and privacy compliance risks.
Selective parameter update methods~\cite{chen2024mofomomentumfilteredoptimizermitigating,hui2024hfthalffinetuninglarge}: These achieve local plasticity control by dynamically identifying task-sensitive parameters. 
Notably, existing research predominantly focuses on the independent optimization of specific technical routes, lacking a unified theoretical explanation of gradient optimization mechanisms. We reveal that existing gradient-oriented methods can be regarded as special cases of this framework under different constraints, thereby uncovering the intrinsic consistency of various CL methods in terms of loss surface geometric properties. This advancement offers a novel analytical foundation for designing theoretically stable continual learning algorithms.

\subsection{Machine Unlearning}
The objective of MU is to eliminate the memory effect of pre-trained models on specific training samples~\cite{shaik2023exploring,xu2024machine,Neel2020DescenttoDeleteGM,Sekhari2021RememberWY}, thereby implementing a data privacy protection at the model~\cite{Neel2020DescenttoDeleteGM,Sekhari2021RememberWY,Ginart2019MakingAF,Guo2019CertifiedDR}. Existing theoretical research has established a mathematical framework for exact MU in convex optimization scenarios from a parameter probability perspective, which has been validated in linear models~\cite{Neel2020DescenttoDeleteGM,Guo2019CertifiedDR,Koh2017UnderstandingBP}. However, when dealing with complex models such as deep neural networks, which are non-convex and high-dimensional, approximate unlearning methods based on gradient optimization have gradually become a mainstream research direction. In recent years, a variety of approximate unlearning methods based on loss function reconstruction have emerged in academia~\cite{Golatkar2019EternalSO,Warnecke2021MachineUO,Jia2023ModelSC,Graves2020AmnesiacML,Thudi2021UnrollingSU,Chundawat2022CanBT}. These methods effectively reduce the risk of privacy leakage caused by target data points by designing regularization terms or constraints with targeted unlearning capabilities. Notably, such unlearning operations can trigger a phenomenon similar to catastrophic forgetting in CL. While ensuring the unlearning of target samples, the model's overall performance on the remaining dataset significantly degrades. Although existing research attempts to maintain the model's generalization ability through fine-tuning on the remaining set, these methods often overlook the analysis of the geometric properties of the optimization path~\cite{Amari1998NaturalGW,Martens2014NewIA,shrestha2023natural}. Differing from traditional gradient update strategies, this study introduces the curvature information of the remaining dataset into the unlearning process for the first time. By systematically analyzing the second-order optimization characteristics in the parameter space, our proposed method can more precisely preserve the model's generalization performance on the original tasks while performing unlearning operations.

\subsection{Coninual Learning-Unlearning}
CLU~\cite{liu2022continual,chatterjee2024unifiedframeworkcontinuallearning} integrates the complementary mechanisms of continual learning and machine unlearning, providing a theoretical framework for knowledge management systems in dynamic environments. This field has gradually gained academic attention in recent years, but existing research still faces significant limitations in system design and unlearning granularity. \cite{Shibata2021LearningWS} first explored the CLU-like system, defining the unlearning process as an adaptive strategy to enhance model plasticity through selective parameter degradation. However, the unlearning behavior in this framework is essentially a byproduct of the CL process, lacking the ability to respond to explicit unlearning instructions. ~\cite{liu2022continual} establishes the first complete formal model of CLU, but to achieve precise unlearning, it requires complete tasks as the minimum unlearning unit and necessitates independent storage of training snapshots for each potentially unlearnable task, leading to exponential storage overhead when faced with large-scale sequential tasks. ~\cite{chatterjee2024unifiedframeworkcontinuallearning} attempts to unify the CL and MU processes by adopting a dual-teacher knowledge distillation architecture to handle knowledge accumulation and unlearning requests separately. However, this design still follows the task-level unlearning assumption as in \cite{Shibata2021LearningWS}, and the alternating iteration mechanism of parallel models significantly increases computational and storage complexity. Notably, unlearning requirements in real-world scenarios are often fine-grained, targeting specific task subsets or individual data points. We pioneer a task-agnostic unified CLU framework, enabling precise responses to unlearning requests of arbitrary scales.
\section{Preliminary}\label{sec: preliminary}

\subsection{Problem Formulation}

In continual learning-unlearning, an agent parameterized by $\theta$ dynamically responds to $T$ sequential task requests $\{\mathcal T_t\}_{t=1}^T$. Each task request $\mathcal T_t=(\mathcal D_t, \mathcal Q_t)$ consists of a tuple of the request type and corresponding dataset: the dataset $\mathcal D_t=\{z_i\}_{i=1}^{N_t}$ contains $N_t$ data points to be either learned or unlearned, including features and labels $z_i=(x_i,y_i)$ in supervised learning. The request type is selected from either learning or unlearning tasks, denoted by $\mathcal Q_t \in\{L, U\}$. We abbreviate learning tasks as $\mathcal T_t^L$, with the corresponding dataset denoted as $\mathcal D_t^L$, and unlearning tasks and datasets as $\mathcal T_t^U$ and $\mathcal D_t^U$, respectively. For learning tasks, we focus on class-incremental learning~\cite{Hsu2018ReevaluatingCL,vandeVen2019ThreeSF} in CL, where the classes between learning tasks do not overlap, and the model must respond to all previously seen classes during prediction. For MU tasks, we focus on removing the impact of a specific class or specific data points. To achieve a dynamic balance between knowledge update and retention, we introduce a replay-based constrained memory buffer, which stores a sampled subset of historical learning task data to approximate the remaining data for replay during tasks. Apart from the explicitly stored subset of historical data in the constrained memory buffer~\cite{Chaudhry2019ContinualLW}, the agent cannot actively access data from previous tasks.

\subsection{Approximate Continual Learning-Unlearning}

Both CL and MU constitute specialized instances within the CLU framework, respectively. Each has corresponding Oracle models that define performance upper bounds for performance pursuit.
In the CL, the joint training model serves as the optimal solution~\cite{Wang2023ACS}, achieving global empirical risk minimization by violating the data isolation principle and accessing the full historical dataset. In the MU, the model retrained on the remaining dataset, excluding all unlearning data, which is ideal for exact unlearning but incurs significant computational costs~\cite{Izzo2020ApproximateDD}. Similarly, in the CLU system, a counterpart exists, namely the optimal model trained on the joint learning data, excluding all unlearning request data. This is formalized as the minimum of empirical loss on all remaining data $\mathcal D^R= (\bigcup_{\{t|\mathcal Q_t=L\}} \mathcal D_t) \backslash ( \bigcup_{\{t|\mathcal Q_t=U\}} \mathcal D_t)$:
\begin{equation}
    \theta_{*}=\mathop{\arg\min}\limits_\theta \sum_{i \in \mathcal{D}^R} \ell\left(\theta ; z_i^R\right)=\mathop{\arg\min}\limits_\theta \mathcal{L}^R(\theta),
\end{equation}
Directly pursuing a model with a parameter distribution identical to the optimal model is intractable. Specifically, joint retraining (Joint-RT) requires storing historical data in the memory buffer and performing full retraining, which incurs unbearable computational resource consumption. Therefore, this study treats such idealized models as gold standards for approximation rather than as comparative baselines. A more practice-oriented optimization paradigm is \textbf{approximate CLU}. When a new task request arrives, we dynamically guide the output distribution of the current model to align with that of the oracle model. Thus, a direct evaluation metric is to quantify the divergence between the model and the optimal model. Additionally, we construct a multi-dimensional evaluation covering CL efficacy, MU effectiveness, knowledge forgetting metrics, and privacy protection strength. For specific benchmark testing schemes and evaluation details, please refer to \textbf{Section.\,\ref{subsec: setting and evaluation}} and \textbf{Appendix.\,C}.

\subsection{Assumptions}

To establish a rigorous theoretical analysis framework, we propose four fundamental assumptions based on common characteristics of practical application scenarios:
\begin{assumption}\label{assump: iid}
    The training and testing data of the model are independently and identically distributed (i.i.d.).
\end{assumption}
This is a fundamental assumption in machine learning~\cite{Ruppert2004TheEO}, ensuring the model's generalization ability during continual learning and the reliability of evaluation on the final test set.
\begin{assumption}\label{assump: local convex}
    The loss function satisfies local convexity conditions within the neighborhood where the model parameters converge to a locally optimal solution.
\end{assumption}
This assumption is rooted in the local approximation theory of non-convex optimization in deep learning~\cite{Choromaska2014TheLS}, providing a theoretical basis for analyzing parameter update directions using second-order Taylor expansions.
\begin{assumption}\label{assump: unlearn history}
    All data samples to unlearn belong to the historical training dataset, i.e., $\mathcal D^U_t \subseteq \bigcup_{i<t} \mathcal{D}_i^{L}$.
\end{assumption}
This assumption avoids the engineering complexity of dynamically verifying data ownership in practical systems, allowing this research to focus on the core analysis of unlearning efficacy at the MU algorithm level.
\begin{assumption}\label{assump: theta k argmin request}
Given the remaining dataset $\mathcal D_t^R$ and the dynamically adjusted dataset $\mathcal D^{\mathcal Q}_t$ (containing incremental or decremental samples) for the $t$-th task, for each iterative parameter $\theta_k$ during the task optimization process, there exists a set of non-negative coefficients $\boldsymbol{\varepsilon}^{\mathcal Q}_k=\{\varepsilon^{\mathcal Q}_{k,i} \}_{i=1}^{N_t},\ \varepsilon^{\mathcal Q}_{k,i}$ that weight the current request data, such that
\begin{equation}
\begin{aligned}
\theta_k &=\mathop{\arg\min}\limits_\theta \sum_{i \in \mathcal{D}^R_t} \ell\left(\theta ; z_i^R\right) + \sum_{j \in \mathcal{D}^{\mathcal Q}_t} \varepsilon^{\mathcal Q}_{k,j}\ell\left(\theta ; z^{\mathcal Q}_j\right)\\
&=\mathop{\arg\min}\limits_\theta \mathcal{L}^R(\theta) + \mathcal L^{\mathcal Q}(\theta;\boldsymbol{\varepsilon}^{\mathcal Q}_k)
\end{aligned}
\end{equation}
\end{assumption}
The theoretical feasibility of this assumption is based on the Gaussian distribution characteristics of gradient noise in deep learning~\cite{Mandt2017StochasticGD}. Through the gradient balance condition $\nabla \mathcal{L}^R(\theta) + \nabla \mathcal L^{\mathcal Q}(\theta;\boldsymbol{\varepsilon}^{\mathcal Q}_k)=0$, an explicit solution for the weighting coefficients $\boldsymbol{\varepsilon}^{\mathcal Q}_k$ satisfying the non-negativity constraint can be constructed. This allows us to infer the next step's direction based on the parameters' optimization properties during the process.

\subsection{Steepest Descent}
Existing CL and MU methods are generally built upon the theoretical framework of gradient updates in parameter space~\cite{Wang2024AUA,Neel2020DescenttoDeleteGM,Thudi2021UnrollingSU}. To establish a unified analytical perspective, we reconstruct the geometric implications of gradient updates based on the \textbf{steepest descent}~\cite{Abrudan2008SteepestDA,Kim2022FisherSI,Martens2014NewIA}~\cite{Kovalev2019StochasticNA}. The goal of the steepest descent is to find the direction $\delta\theta=\theta_{k+1}-\theta_k$ that drives the objective function $F(\theta)$ to descend the fastest within a $\xi$-neighborhood of the current parameters $\theta_k$. This can be formulated as the following optimization problem. (See \textbf{Appendix.\,A.1} for proof.)
\begin{equation}
\begin{aligned}
\delta \theta &:=\mathop{\arg\min}_{\rho(\theta_k,\theta_k+\delta \theta) \leq \xi} F(\theta_k+\delta \theta) \\ 
\Rightarrow \theta_{k+1} &:=\mathop{\arg\min}_{\theta_{k+1}} F(\theta_{k+1}) + \frac{1}{\alpha_k(\xi,\theta_k)}\rho(\theta_{k},\theta_{k+1}),
\end{aligned}
\end{equation}
where $\rho(\cdot,\cdot)$ represents the manifold metric that defines the geometry of the parameter’s neighborhood. To simplify the derivation, we rewrite it to the form on the right, where $\alpha_t(\xi,\theta_t)$ represents the learning rate required to move the distance $\xi$. In the following, we fix a small learning rate to approximate a search within a local neighborhood. The characterization $\rho$ of the underlying coordinate space of the neighborhood will determine update directions, optimization paths, and the flatness of the minimum. Vanilla gradient descent is obtained using the Euclidean metric, Newton’s direction is measured by the second-order expansion of the objective function~\cite{Kovalev2019StochasticNA}, and the KL divergence in the output space induces a natural gradient~\cite{Martens2014NewIA,Calin2014GeometricMI}. Steepest descent has inspired continual learning in previous works. Next, we will explore approximate CLU through vanilla gradient descent and attempt to benefit from improved manifold metrics.
\section{A Unified Gradient-based Framework for Continual Learning-Unlearning}\label{sec: theory}

\subsection{Optimization Problem}

The optimization objective of approximate CLU aims to minimize the discrepancy between the current model's output distribution and that of the Joint-RT model. To achieve this goal, we unify knowledge incremental learning and decremental unlearning within the same framework, associating each task request with its corresponding ideal reference model. Precisely, without loss of generality, \textit{we can formalize the problem by omitting the task subscript $t$ and just considering a scenario involving two consecutive task requests}: let the optimal model parameters $\theta_0$ for the preceding task be defined as the solution obtained from joint training on the remaining data $\mathcal D^R$ and the data to be unlearned $\mathcal D^U$, i.e., $\theta_0=\mathop{\arg\min}_{\theta}\mathcal L^R(\theta) + \mathcal L^U(\theta)$. Meanwhile, the optimal reference parameters $\theta_*$ for the current task are the joint optimization solution corresponding to the remaining data $\mathcal D^R$ and the newly learned data $\mathcal D^L$ after removing $\mathcal D^U$, i.e., $\theta_*=\mathop{\arg\min}_{\theta}\mathcal L^R(\theta) + \mathcal L^L(\theta)$. We employ the KL divergence $D_{\mathrm{KL}}(\cdot,\cdot)$ to quantify the difference in model output distributions. Thus, solving approximate CLU is equivalent to iteratively updating the model parameters $\theta_k$ starting from the initial parameters $\theta_0$ along the steepest descent direction, gradually reducing the output distribution discrepancy with the ideal model $\theta_*$. The formulation of this optimization problem can be expressed as follows:
\begin{equation}\label{eq: optimization problem}
\begin{aligned}
\theta_{k+1}
=&\mathop{\arg\min}\limits_{\theta_{k+1}} D_{\text{KL}}\left(p_z(\theta_*)||p_z(\theta_{k+1})\right)+\frac{1}{\alpha_k}\rho(\theta_k,\theta_{k+1})\\
=&\mathop{\arg\min}\limits_{\theta_{k+1}}
\underbrace{D_{\text{KL}}\left(p_{z^R}(\theta_*)||p_{z^R}(\theta_{k+1})\right)}_{(a)}p^R\\ 
&+\underbrace{D_{\text{KL}}\left(p_{z^L}(\theta_*)||p_{z^L}(\theta_{k+1})\right)}_{(b)}p^L\\
&+\underbrace{D_{\text{KL}}\left(p_{z^U}(\theta_*)||p_{z^U}(\theta_{k+1})\right)}_{(c)}p^U\\ 
&+\frac{1}{\alpha_k}\underbrace{\rho(\theta_k,\theta_{k+1})}_{(d)}, 
\end{aligned}
\end{equation}
where $p_z(\theta) = p(z; \theta)$ represents the model output probability, and $p^R= p(\mathcal D^R)$, $p^L= p(\mathcal D^L)$, $p^U= p(\mathcal D^U)$ denote the partitions of the remaining, learning, and unlearning data, respectively. In class incremental classification problems, models output the class posterior probability $p(z ; \theta) = p(y | x; \theta)$, and the empirical risk for each sample is the cross-entropy loss. Analyzing the terms in \eqref{eq: optimization problem}:
\begin{itemize}
    \item[$\bullet$] ($a$) ensures the model's foundational capability by maintaining its performance on the remaining data $\mathcal D^R$;
    \item[$\bullet$] ($b$) and ($c$) correspond to the task requests for incremental learning $\mathcal D^L$ and unlearning $\mathcal D^U$, respectively;
    \item[$\bullet$] ($d$) employs the metric $\rho$ to constrain the magnitude of each update, thereby identifying the direction of the steepest descent on the manifold.
\end{itemize}
This modular design offers two theoretical advantages: First, although the initial modeling considers only adjacent task pairs, the actual optimization process can be recursively extended to task sequences of arbitrary length from it. Second, the framework does not presuppose specific implementations of learning or unlearning, demonstrating strong modality independence and application flexibility.

\begin{figure}[tb]
    \centering
    \includegraphics[width=0.48\textwidth]{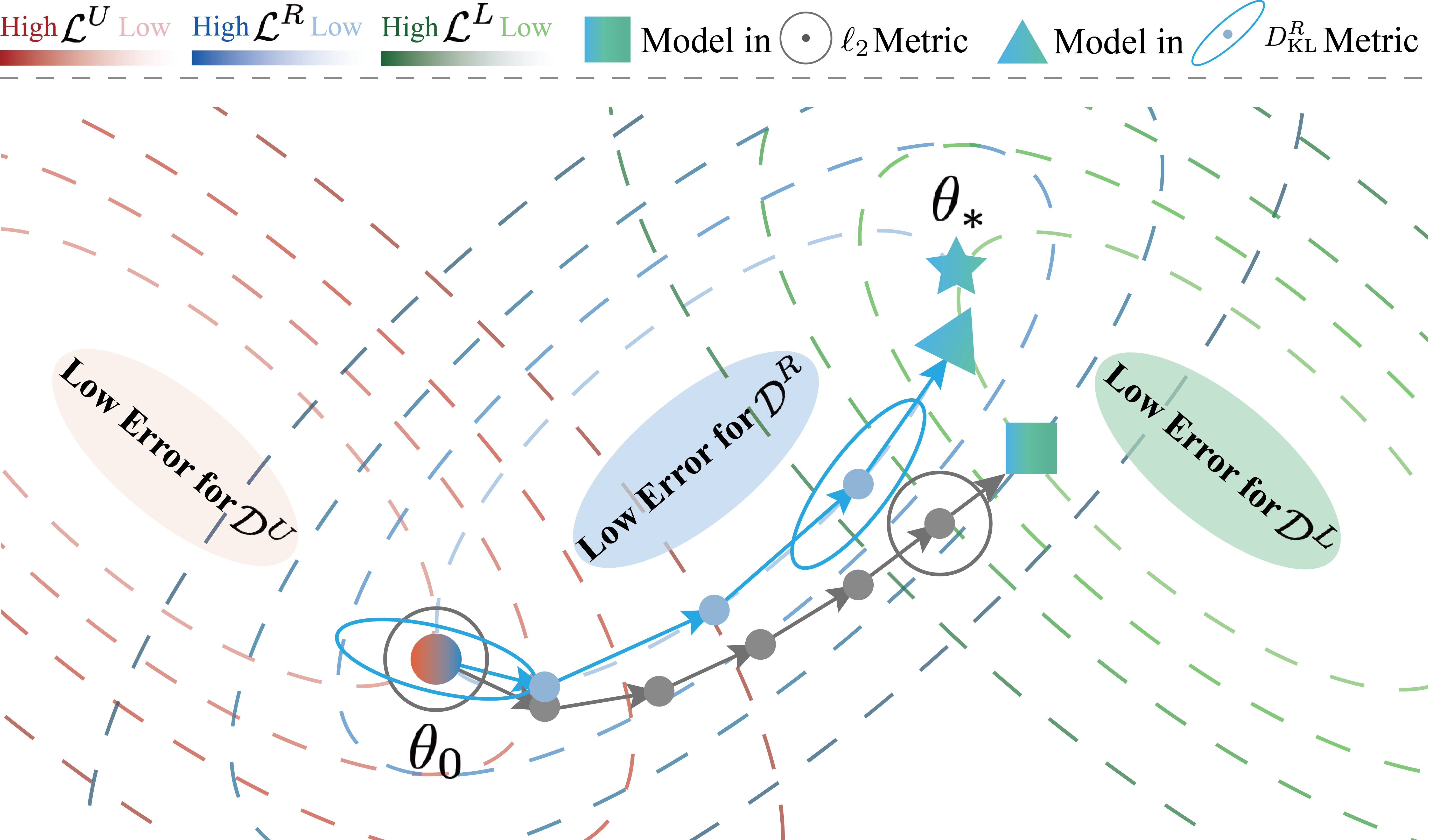}
    \caption{The visualization diagram for the optimization process of the proposed approximate CLU. We model the trajectory starting from the initial model $\theta_0$, which is well-trained on both the unlearning data $\mathcal D^U$ and the remaining data $\mathcal D^R$, guiding the model's output distribution toward the optimal model $\theta_*$ that achieves low error on both the newly learning data $\mathcal D^L$ and the remaining data $\mathcal D^R$. For the solution, we not only derive the vanilla gradient descent optimization trajectory (gray arrow trajectory in the figure, \textbf{Proposition.\,\ref{prop: clu vanilla}}) based on the Euclidean $\ell_2$ metric, but also further consider the optimization trajectory corrected by the remaining Hessian under the remain-preserved manifold $D^R_{\mathrm{KL}}$ metric (blue arrow trajectory in the figure, \textbf{Proposition.\,\ref{prop: clu remain}}). This effectively balances the efficacy of knowledge updates with the preservation of model generalization performance, significantly reducing performance degradation on the remaining data $\mathcal D^R$.}
    \label{fig: manifold}
\end{figure}

\subsection{Vanilla Gradient Descent}

Although \eqref{eq: optimization problem} has an intuitive theoretical form, its direct solution faces a fundamental obstacle—we cannot directly access the parameters $\theta_*$ of the ideal reference model (oracle model). Therefore, we turn to optimality condition analysis and Taylor expansion at $\theta_*$, attempting to derive feasible optimization directions for approximate CLU. Based on the theoretical insight from \textbf{Assumption\,\ref{assump: theta k argmin request}}, the current model parameters $\theta_k$ can be analytically interpreted as specific sampling points on the continuous optimization trajectory connecting the initial parameters $\theta_0$ and the ideal parameters $\theta_*$. This geometric perspective can be formally expressed as:
\begin{equation}\label{eq: theta k argmin lu}
\theta_{k} =\mathop{\arg\min}\limits_\theta \mathcal{L}^R(\theta) + \mathcal L^{L}(\theta;\boldsymbol{\varepsilon}^{L}_{k})+ \mathcal L^{U}(\theta;\boldsymbol{\varepsilon}^{U}_{k})
\end{equation}
where the boundary conditions for the weighting coefficients are as follows: when $\theta_{k}=\theta_0$, $\boldsymbol \varepsilon^L_{k}=\boldsymbol 0,\boldsymbol \varepsilon^U_{k}=\boldsymbol 1$; and when $\theta_{k}=\theta_*$, $\boldsymbol \varepsilon^L_{k}=\boldsymbol 1,\boldsymbol \varepsilon^U_{k}=\boldsymbol 0$, with $\boldsymbol 0$ and $\boldsymbol 1$ representing all-zero and all-one sets, respectively. Leveraging the local convexity of the loss function at the optimal parameter point (as stated in \textbf{Assumption\,\ref{assump: local convex}}), we perform a Taylor expansion of the parameters and utilize the property that the optimality gradient is zero. This allows us to determine the vanilla gradient descent for approximate CLU by using the Euclidean distance $\ell_2$ as the manifold metric, as stated in \textbf{Proposition.\,\ref{prop: clu vanilla}}:
\begin{prop}\label{prop: clu vanilla}
Under the Euclidean manifold metric, $\rho(\theta_k,\theta_{k+1})=\frac12\lVert \theta_k-\theta_{k+1} \rVert^2$. Assuming that the current model satisfies \eqref{eq: theta k argmin lu}. Let $H_*^R=\nabla^2\mathcal L^R(\theta_*)$ denote the Hessian of the oracle model on the remaining set and $H_{k}^L=\nabla^2\mathcal L^L(\theta_{k}), H_{k}^U=\nabla^2\mathcal L^U(\theta_{k})$ denote the Hessian of the current model on the request sets, respectively. Then, the steepest descent direction that minimizes \eqref{eq: optimization problem} is approximately:
\begin{equation}\label{eq: clu vanilla}
\begin{aligned}
    \theta_{k+1}-\theta_k:\approx-\alpha_k
    [
    &\underbrace{\nabla \mathcal L^R(\theta_k)}_{(R)}p^R\\
    +\underbrace{\frac12(H^L_kp^L + H^U_kp^U)(H_*^R)^{-1}}_{(S)}[&\underbrace{\nabla \mathcal L^L(\theta_k;\boldsymbol{1}-\boldsymbol{\varepsilon}^L_k)}_{(L)}] \\
    + \underbrace{\frac12(H^L_kp^L + H^U_kp^U)(H_*^R)^{-1}}_{(S)}[&\underbrace{\nabla \mathcal L^U(\theta_k;-\boldsymbol{\varepsilon}^U_k)}_{(U)}]
    ].
\end{aligned}
\end{equation}
\end{prop}
The proof can be found in \textbf{Appendix.\,A.2}. To elucidate the effectiveness of the gradient-based methods, we decompose this vanilla gradient descent direction in \eqref{eq: clu vanilla} into four components: ($R$), ($L$), ($U$), and ($S$).
\begin{itemize}
    \item[$\bullet$] (\textbf{R}emaining): Theoretically, this corresponds to the gradient descent direction for the entire historical remaining data $\mathcal D^R$. However, constrained by the data access limitations of CLU, we approximate this by maintaining a subset of historical data through a dynamic memory buffer, thereby mitigating catastrophic forgetting during task iterations.
    \item[$\bullet$] (\textbf{L}earning) and (\textbf{U}nlearning): These represent the gradient directions for the current learning task and the task to be unlearned, respectively. Notably, the loss weighting coefficient for the unlearning term is negative, effectively constituting a gradient ascent update for $\mathcal D^U$, i.e., eliminating the influence of these specific samples.
    \item[$\bullet$] (\textbf{S}aliency): An adaptive adjustment term composed of three parts of the Hessian matrix, where $H_k^L$ and $H_k^U$ amplify parameter updates critical to the current task, while $(H_*^R)^{-1}$ suppresses update components that may harm the retention of historical knowledge. This is similar to the parameter importance components increasingly emphasized in recent CL and MU works~\cite{chen2024mofomomentumfilteredoptimizermitigating,Fan2023SalUnEM}, serving as a dynamic arbitrator between knowledge updating and retention.
\end{itemize}
Thus, we present the vanilla gradient descent iteration method for general CLU. This structured decomposition not only reveals the intrinsic connections between the CLU framework and classical CL and MU methods but, more importantly, provides clear guiding principles for designing theoretically guaranteed universal CLU algorithms. By appropriately configuring the functional implementation of each component, it can flexibly adapt to the requirements of different application scenarios.

\subsection{Improvement in Remain-preserving Manifold}

In fact, employing Euclidean distance as the manifold metric for parameter updates is arbitrary. It treats all coordinates as equally important because the local second-order expansion is identical, $\nabla^2_{\theta_t}(\frac12\lVert \theta_t-\theta_{t+1} \rVert^2)=I$. This homogeneous metric overlooks the critical differences in parameters' importance for task knowledge retention, particularly in addressing the pervasive issue of catastrophic forgetting in CLU, where an urgent need is to construct a specialized manifold structure in the parameter space. The core characteristic of this manifold lies in ensuring that, even when parameters undergo significant displacements in Euclidean space, the model's performance on previously acquired knowledge remains stable~\cite{Hoang2023LearnTU}. Since directly predicting the impact of new tasks on existing knowledge is theoretically infeasible, this study proposes achieving knowledge stability by constructing a remain-preserving manifold. Specifically, we require that the parameter update process always resides within the manifold space that minimally affects the output distribution of the remaining set. An empirical characterization of such a manifold could be \textit{the KL divergence on the output distribution of the remaining set}, $D_{\mathrm{KL}}^R$. Given that the models before and after the task are both well-trained on the remaining data, their outputs are very close to the ground-truth remaining distribution. By starting with $\theta_0$ and limiting updates to this manifold, the maintained output distribution keeps almost consistent throughout the update iterations, $\nabla\mathcal L^R(\theta_0)\approx\nabla\mathcal L^R(\theta_*)\approx0$. This consistency allows for a second-order Taylor expansion at $\theta_k$ for terms ($a$) and ($d$) in \eqref{eq: optimization problem}, providing crucial curvature information for task updates to prevent deviations in the model output on the remaining set, leading to \textbf{Proposition.\,\ref{prop: clu remain}}:
\begin{prop}\label{prop: clu remain}
    Using the model output KL divergence on the remaining set as the manifold metric, $\rho(\theta_k,\theta_{k+1})=D^R_{\text{KL}}\left(p_{z^R}(\theta_k)||p_{z^R}(\theta_{k+1}))\right)$. Assuming that the current model satisfies \eqref{eq: theta k argmin lu}. Let $H_k^R=\nabla^2\mathcal L^R(\theta_k)$ represent the Hessian w.r.t. $\theta_k$ on the remaining set, then the steepest descent direction that minimizes \eqref{eq: optimization problem} is approximately:
\begin{equation}\label{eq: clu remain}
\begin{aligned}    
\Rightarrow\theta_{k+1}-\theta_k:\approx 
&-\frac{\alpha_k}{p^R+1}        
\underbrace{(H_k^R)^{-1}}_{(R)} \cdot\\          
&\underbrace{\frac12(H^L_kp^L + H^U_kp^U)(H_*^R)^{-1}}_{(S)}\cdot\\    
&[
    \underbrace{\nabla \mathcal L^L(\theta_k;\boldsymbol{1}-\boldsymbol{\varepsilon}^L_k)}_{(L)} + 
    \underbrace{\nabla \mathcal L^U(\theta_k;-\boldsymbol{\varepsilon}^U_k)}_{(U)}      
].
\end{aligned}
\end{equation}
\end{prop} 
The proof is deferred to \textbf{Appendix.\,A.3}. The task update in \eqref{eq: clu remain} incorporates second-order Hessian information related to the remaining data to guide the optimization direction. Specifically, the large curvature directions of $H_k^R$ correspond to the weights that encapsulate retained knowledge, while the small curvature directions encourage model updates for effective knowledge adaptation.

\begin{table}[ht]
        \centering
         \caption{\footnotesize{Comparison of CL, MU, and CLU methods. We decompose the steepest descent direction for CLU into four parts: the remaining part ($R$), the learning part ($L$), the unlearning part ($U$), and the weight saliency mask ($S$) as in \eqref{eq: clu vanilla} and \eqref{eq: clu remain}. Our method consider all the components within the remain-preserving manifold.
         }}
        \label{tab: revisit_CL_MU}
        \resizebox{0.49\textwidth}{!}{
        \begin{tabular}{l|cc|cccc|c}
        \toprule
        \multirow{2}{*}{\textbf{Methods}} & \multicolumn{2}{c|}{\textbf{Task}} & \multicolumn{4}{c|}{\textbf{CLU components}} & \textbf{Manifold} \\
         & CL & MU & ($R$) & ($L$) & ($U$) & ($S$) & \textbf{Metric} \\
         \midrule
        FT~\cite{Warnecke2021MachineUO} & \checkmark & \checkmark & & \checkmark & \checkmark & & ${\ell_2}$  \\
        \midrule
        LwF~\cite{Li2016LearningWF} & \checkmark & & \checkmark & \checkmark & & & ${\ell_2}$  \\
        ER~\cite{Chaudhry2019ContinualLW} & \checkmark & & \checkmark & \checkmark & & & ${\ell_2}$  \\
        DER++~\cite{Buzzega2020DarkEF} & \checkmark & & \checkmark & \checkmark & & & ${\ell_2}$  \\
        EWC~\cite{Kirkpatrick2016OvercomingCF} & \checkmark & & \checkmark & \checkmark & & & $D^R_{\text{KL}}$  \\
        MoFO~\cite{chen2024mofomomentumfilteredoptimizermitigating} & \checkmark & & & \checkmark & &\checkmark & ${\ell_2}$ \\
        \midrule
        GA~\cite{Graves2020AmnesiacML,Thudi2021UnrollingSU} & & \checkmark & & & \checkmark & & ${\ell_2}$  \\
        NegGrad+~\cite{Kurmanji2023MachineUI} & & \checkmark & \checkmark & & \checkmark & & ${\ell_2}$  \\
        SCRUB~\cite{Kurmanji2023TowardsUM} & & \checkmark & \checkmark & & \checkmark & & ${\ell_2}$ \\
        L2UL~\cite{Cha2023LearningTU} & & \checkmark &  \checkmark & & \checkmark & & $D^R_{\text{KL}}$ \\
        SalUn~\cite{Fan2023SalUnEM} & & \checkmark & \checkmark & &\checkmark &\checkmark & ${\ell_2}$  \\
        \midrule
        UniCLUN~\cite{chatterjee2024unifiedframeworkcontinuallearning} & \checkmark & \checkmark & \checkmark & \checkmark &\checkmark & & ${\ell_2}$ \\
        \rowcolor{gray!20}
        UG-CLU & \checkmark & \checkmark & \checkmark & \checkmark &\checkmark & \checkmark& $D^R_{\text{KL}}$ \\
        \bottomrule
        \end{tabular}}
\vspace*{-3mm}
\end{table}
\subsection{Revisit Gradient-based CL and MU methods in a Unified Framework}

In practical applications, we can allow the model to handle continual learning and unlearning task requests sequentially. The aforementioned composite CLU problem can be decomposed into two fundamental scenarios: pure learning tasks ($\mathcal D^U=\emptyset$) and pure unlearning tasks ($\mathcal D^L=\emptyset$). By setting the corresponding datasets to the empty set $\emptyset$, the dynamic iterations of the CLU system will reduce to the descriptions of standard CL and MU, respectively. From \textbf{Proposition.\,\ref{prop: clu remain}}, we derive the following two simplified corollaries:
\begin{corollary}\label{coro: cl remain} (CL)
    Using the model output KL divergence on the remaining set as the manifold metric, $\rho(\theta_k,\theta_{k+1})=D^R_{\text{KL}}\left(p_{z^R}(\theta_k)||p_{z^R}(\theta_{k+1}))\right)$. Assuming that the current model satisfies \eqref{eq: theta k argmin lu}. Let $\tilde\alpha^L_k=\alpha_k p^L/(\alpha_kp^R+1)$, then the steepest descent direction that minimizes \eqref{eq: optimization problem} is approximately:
\vspace{-1mm}
\begin{equation}\label{eq: cl remain}
\theta_{k+1}-\theta_k:\approx 
-\tilde\alpha^L_k        
\underbrace{(H_k^R)^{-1}}_{(R)}            
[
\underbrace{H^L_k(H_*^R)^{-1}}_{(S)}
\underbrace{\nabla \mathcal L^L(\theta_k;\boldsymbol{1}-\boldsymbol{\varepsilon}^L_k)}_{(L)}     
].
\end{equation}
\end{corollary} 
\begin{corollary}\label{coro: mu remain} (MU)
    Using the model output KL divergence on the remaining set as the manifold metric, $\rho(\theta_k,\theta_{k+1})=D^R_{\text{KL}}\left(p_{z^R}(\theta_k)||p_{z^R}(\theta_{k+1}))\right)$. Assuming that the current model satisfies \eqref{eq: theta k argmin lu}. Let $\tilde\alpha^U_k=\alpha_k p^U/(\alpha_k p^R+1)$, then the steepest descent direction that minimizes \eqref{eq: optimization problem} is approximately:
\vspace{-1mm}
\begin{equation}\label{eq: mu remain}
\theta_{k+1}-\theta_k:\approx 
-\tilde\alpha^U_k       
\underbrace{(H_k^R)^{-1}}_{(R)}            
[
    \underbrace{H^U_k(H_*^R)^{-1}}_{(S)}   
    \underbrace{\nabla \mathcal L^U(\theta_k;-\boldsymbol{\varepsilon}^U_k)}_{(U)}      
].
\end{equation}
\end{corollary} 

By systematically integrating \textbf{Proposition.\,\ref{prop: clu vanilla},\,\ref{prop: clu remain}} and \textbf{Corollary.\,\ref{coro: cl remain},\,\ref{coro: mu remain}}, we encapsulate existing gradient-based CL and MU methods within a unified theoretical framework, as summarized in \textbf{Table.\,\ref{tab: revisit_CL_MU}}. Specifically, replay-based methods explicitly optimize to retain preserved knowledge (e.g., ER~\cite{Chaudhry2019ContinualLW}, DER++~\cite{Buzzega2020DarkEF}, NegGrad+~\cite{Kurmanji2023MachineUI}); classic regularization-based CL/MU methods (e.g., EWC~\cite{Kirkpatrick2016OvercomingCF}, L2UL~\cite{Cha2023LearningTU}) maintain the model's general capability by approximating the second-order Hessian matrix $H_k^R$. Recent parameter selection methods (e.g., MoFO~\cite{chen2024mofomomentumfilteredoptimizermitigating}, SalUn~\cite{Fan2023SalUnEM}) enhance learning/unlearning efficiency while effectively mitigating knowledge degradation through salient parameter identification mechanisms. Although the existing CLU method UniCLUN~\cite{chatterjee2024unifiedframeworkcontinuallearning} proposed a unified solution for learning and unlearning, it lacks parameter saliency modeling and its optimization path is suboptimal. However, few approaches can fully cover our decomposition of approximate CLU. We propose a salient parameter approximation update mechanism based on the proposed remain-preserving manifold, ensuring continuous knowledge evolution while achieving efficient optimization.

 

\section{Our Method: UG-CLU}\label{sec: method}

Based on the optimization iteration derived from the theoretical analysis in \textbf{Section.\,\ref{sec: theory}}, this section systematically elaborates on the implementation framework of the method. \textit{To simplify the presentation, this section primarily focuses on the single-task optimization process, omitting the task subscript $t$}. Considering the discrepancies between practical application scenarios and theoretical assumptions, we construct a feasible implementation scheme through four key modules: First, a dynamic reservoir sampling buffer is employed to achieve an approximation of the historical data distribution. Second, the proposed fast-slow parameter update method is introduced to implicitly approximate the optimization direction modulated by the online Hessian, effectively avoiding the computational overhead associated with explicit Hessian matrix calculations. Simultaneously, a sample-wise adaptive weighting coefficient is designed to dynamically adjust the optimization weights of the task loss function. Finally, the balanced weight saliency method is utilized to select and update critical parameters.

\begin{figure*}[tb]
    \centering
    \includegraphics[width=\textwidth]{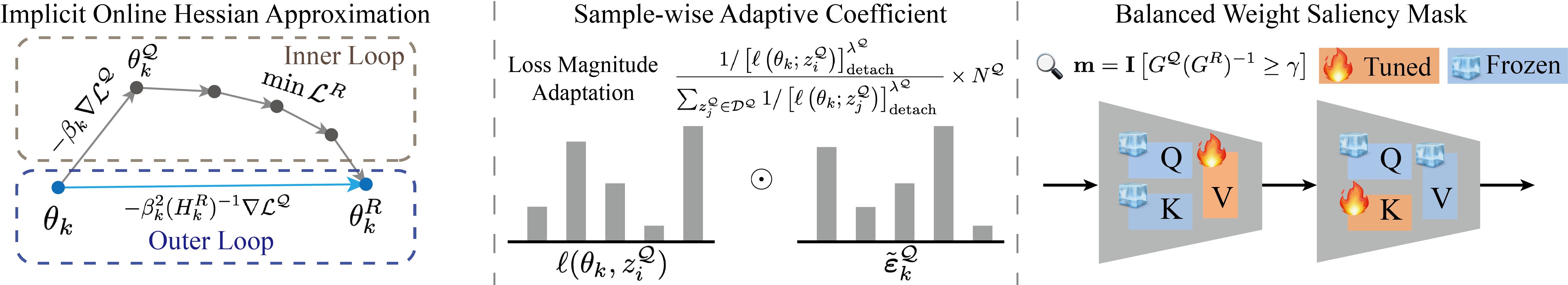}
    \caption{Three core modules of our proposed \textbf{UG-CLU}. The \textbf{implicit online Hessian approximation} derives the task optimization direction modulated by the Hessian through a fast-slow weight update mechanism, reducing computational overhead. The \textbf{sample-wise adaptive coefficient} leverages task properties and loss magnitude adaptation to re-weight the sample loss, effectively approximating the optimality conditions of the iterative process. The \textbf{balanced weight saliency mask} simultaneously considers the efficiency of task knowledge updates and the effectiveness of knowledge retention, selecting the most important parameters for updates. The combination of these three modules aligns with the results derived from our theoretical analysis, achieving approximate CLU.}
    \label{fig: method}
\end{figure*}

\subsection{Reservoir Memory Buffer}

Replay-based CL methods~\cite{Chaudhry2019OnTE,Buzzega2020DarkEF,Rebuffi2016iCaRLIC,lopez2017gradient} typically retain samples from historical tasks through a memory buffer. The core principle lies in utilizing the buffered data to maintain the statistical characteristics of the historical distribution, $\mathcal B^R\subset\mathcal D^R$, and in this section, $\mathcal L^R(\theta)=\sum_{i \in \mathcal{B}^R} \ell\left(\theta ; z_i^R\right)$. By jointly optimizing the current task data and replay samples, knowledge consolidation is achieved, thereby mitigating catastrophic forgetting. In the CLU system, we inherit this foundational framework and adapt it to different task types: (1) During the learning task phase, a dynamic reservoir sampling strategy~\cite{Vitter1985RandomSW,Chaudhry2019OnTE,Riemer2018LearningTL} is employed to update the buffer, ensuring a smooth transition of data distribution; (2) During the unlearning task phase, data samples related to the unlearning target are promptly erased from the buffer, preventing their gradient updates in subsequent training. It is important to note that although existing research has proposed various optimization methods for memory buffers, such as generative data augmentation~\cite{Shin2017ContinualLW} and importance sampling mechanisms~\cite{lopez2017gradient}, given that CLU systems are still in their infancy, this paper primarily focuses on validating the effectiveness of the basic architecture. Therefore, a minimally designed storage strategy is adopted. Efficient buffer optimization for CLU scenarios, including storage efficiency improvement and privacy-utility trade-offs, will be explored in future work.

\subsection{Implicit Online Hessian Approximation}\label{subsec: implicit hessian}

\textbf{Challenges in Hessian Approximation}.
To exploit the benefits of general CLU updates within the remain-preserving manifold, the key point lies in $(H_k^R)^{-1}$.
However, calculating the Hessian and its inverse for large-scale models is computationally demanding~\cite{elsayed2022hesscale}. 
Consequently, many methods have been developed to estimate the Hessian, such as Fisher information~\cite{Kirkpatrick2016OvercomingCF}, Fisher diagonals~\cite{Husz_r_2018}, and Kronecker-factored Laplace approximation~\cite{Ritter2018OnlineSL}. 
However, most of these methods compute a fixed Hessian before the task begins, which leads to progressively increasing estimation biases. These biases are exacerbated by cumulative errors in Taylor expansion, ultimately harming the retained performance. 

\textbf{Fast-slow weight update}. Given that computing the inverse of even well-approximated Hessian demands substantial computational resources, it is more practical to estimate the optimization direction post-Hessian inversion modulation directly. Inspired by recent insights into the connection between Meta-Continual Learning and Hessian approximation~\cite{wu2024meta}, we propose a fast-slow weight method~\cite{Zhang2019LookaheadOK,Nichol2018OnFM} for implicitly approximating the desired updates. The optimization problem for fast weight updates is formulated as follows:
\begin{equation}\label{eq: fast update}
\mathop{\min}_{\theta^{\mathcal Q}_{k}}\mathcal L^R(\theta_k^{\mathcal Q})\quad \text{s.t.}\quad \theta_k^{\mathcal Q}=\theta_k-\beta_k\nabla \mathcal L^{\mathcal Q}(\theta_k),
\end{equation}
where $\nabla \mathcal L^{\mathcal Q}(\theta_k)$ represents the gradient update for the learning task $\nabla \mathcal L^L(\theta_k;\boldsymbol{1}-\boldsymbol{\varepsilon}^L_k)$ or the unlearning task $\nabla \mathcal L^U(\theta_k;-\boldsymbol{\varepsilon}^U_k)$, and $\beta_k$ denotes its learning rate. The iterative process is depicted in \textbf{Figure.\,\ref{fig: method}}. A step of task update is taken at the current model, resulting in $\theta_k^{\mathcal Q}$. Several gradient descent updates on the remaining set follow to obtain the minimum point $\theta_k^R$. This fine-tuning ensures that the updated model adheres to the remain-preserving manifold. The slow weight updates leverage the underlying connection between $\theta_k$ and $\theta_k^R$, as stated in \textbf{Proposition.\,\ref{prop: implicit hessian approx}}:
\begin{prop}\label{prop: implicit hessian approx} 
For implicit online Hessian approximation in \eqref{eq: fast update}, suppose $\beta_k,\delta_k$ is small, $\beta_k<\sqrt{\delta_k/|\nabla \mathcal L^R(\theta_k)-[\nabla \mathcal L^R(\theta_k)]^2|}$, $\mathcal L^R$ is $\mu$-smooth, i.e., $\lVert \nabla \mathcal L^R(\theta) - \nabla \mathcal L^R(\theta') \rVert_2\leq \mu\lVert\theta-\theta' \rVert_2$, and there exist an $\zeta_k$-neighborhood $\mathcal N(\theta_k^R,\zeta_k)$ of the optimal model parameter $\theta_k^R=\mathop{\arg\min}_{\theta^{\mathcal Q}_{k}}\mathcal L^R(\theta^{\mathcal Q}_k)$, which includes $\theta_k$ and $\theta_k^{\mathcal Q}$. Then, the iterative update term approximately is,
\begin{equation}\label{eq: implicit hessian update}
\begin{aligned}
    \theta_k-\theta_k^R :\approx &\beta_k^2 \left[ \nabla^2\mathcal L^R(\theta_k)\right]^{-1}\nabla \mathcal L^{\mathcal Q}(\theta_k)\\
    =&\beta_k^2 (H_k^R)^{-1}\nabla \mathcal L^{\mathcal Q}(\theta_k).
\end{aligned}
\end{equation}
\end{prop}
The proof is in \textbf{Appendix.\,A.4}. \textbf{Proposition.\,\ref{prop: implicit hessian approx}} indicates that the model $\theta_k^R$, obtained after fine-tuning using the process described in \eqref{eq: fast update}, is approximately equivalent to updating the current model $\theta_k$ by one step in the Hessian-adjusted task optimization direction. We use this direction to update the outer loop.

\textbf{Comparison to joint optimization loss}. To provide an intuitive explanation of the advantages of implicit online Hessian approximation, we investigate the differences in the updates between our optimization in \eqref{eq: fast update} and the joint optimization of task update and remaining losses, as seen in traditional replay-based CL and MU methods. We take the checkpoint after the first step of fine-tuning the remaining set as an example and ignore the step size.
\begin{equation}
\mathcal L^{\text{Replay}}(\theta_k)=\mathcal L^{\mathcal Q}(\theta_k)+\mathcal    L^{R}(\theta_k),
\end{equation}
\begin{equation}\label{eq: delta replay}
\Delta^{\text{Replay}}=\nabla\mathcal L^{\mathcal Q}(\theta_k)+\nabla\mathcal L^R(\theta_k),
\end{equation}
\begin{equation}
\begin{aligned}
\Delta^{\text{our}}\label{eq: delta our} 
\approx&\nabla\mathcal L^{\mathcal Q}(\theta_k)+\nabla\mathcal L^R(\theta_k)+\nabla^2\mathcal L^R(\theta_k)(\theta_k^{\mathcal Q}-\theta_k)    \\
=&(I-H^R_k)\nabla\mathcal L^{\mathcal Q}(\theta_k)+\nabla\mathcal L^R(\theta_k),
\end{aligned}
\end{equation}
Comparison of the updates in \eqref{eq: delta replay} and \eqref{eq: delta our} reveals that the remaining gradient is the same. Our task update in fast weight is adjusted by an additional term $-H^R_k$, which is absent in joint optimization. This modification weakens the directions that significantly impact the remain, thereby mitigating the damage of task loss on the retained performance. 

It is noteworthy that some existing CL and MU methods are intrinsically related to our approach at the optimization level. Specifically, the bi-level optimization framework proposed in CL models sequential task updates as an alternating inner-outer optimization problem~\cite{shaker2020bilevelcontinuallearning}. Meta-continual learning methods achieve rapid model adaptation through inner-loop hyper-gradient updates~\cite{Nichol2018OnFM,wu2024meta}. The typical two-phase unlearning mechanism in MU suggest first impairing the model and then repairing its performance~\cite{Tarun2021FastYE}, paralleling a single fast-slow weight update mechanism in our method.

\subsection{Sample-wise Adaptive Coefficient}\label{subsec: adaptive coef}

Based on the theoretical derivation, this paper proposes a parameter-adaptive weighted gradient update strategy. As shown in \eqref{eq: cl remain} for learning and \eqref{eq: mu remain} for unlearning tasks, the determination of the gradient update direction relies on the precise calculation of the weighting coefficients $\boldsymbol{\varepsilon}^L_k$ and $\boldsymbol{\varepsilon}^U_k$, which are derived from the minimization problem of $\theta_k$ in \eqref{eq: theta k argmin lu}. However, due to the computational intractability of solving the $\arg\min$ problem inversely, we constructs a feasible path between theoretical derivation and practical feasibility by analyzing the coefficient characteristics and the optimization objectives, thereby proposing a heuristic approximation strategy. It is important to emphasize that the learning and unlearning processes mathematically correspond to exploration and rectification behaviors in the parameter space, leading to significantly distinct characteristics in the two types of weighting coefficients.

\textbf{CL} For the weighting coefficients $\boldsymbol{\varepsilon}^L_k$ in learning tasks, the following conditions are satisfied: \ding{172} For the initial model, $\boldsymbol{\varepsilon}^L_0=\boldsymbol{0}$. \ding{173} For the final optimal model, $\boldsymbol{\varepsilon}^L_*=\boldsymbol{1}$. \ding{174} Assuming homogeneity among samples~\cite{Rozemberczki2022TheSV}, if $\ell(\theta ; z_i^L)>\ell(\theta ; z_j^L)$, then $\varepsilon_{k,i}<\varepsilon_{k,j}$. Moreover, considering the continuity in the model parameter space implies the continuity of $\boldsymbol{\varepsilon}^L_k$ in the function space. Therefore, our heuristic estimation of the weighting coefficients for learning tasks includes an increment with iteration steps and sample-wise adaptation based on loss magnitude.
\begin{equation}\label{eq: cl coefficient}
\tilde\varepsilon^L_{k,i}=\min\left(1, \frac{k}{K}\frac{1/\mathcal [\ell(\theta_k;z_i^L)]_{\text{detach}}^{\lambda^L}}{\sum_{z^L_j\in\mathcal D^L}1/[\ell(\theta_k;z^L_j)]_{\text{detach}}^{\lambda^L}}\times N^L\right)
\end{equation}
where $1\leq i \leq N^L$, $K$ represents the outer loop iteration, $\lambda^L$ is the temperature scalar to control the smoothness of the coefficients, and $[\cdot]_{\text{detach}}$ denotes the operation to detach a tensor from the computational graph, meaning that $\ell(\cdot)$ in \eqref{eq: cl coefficient} only serves for weighting coefficients and does not contribute to the gradient. During actual training, the parameter updates are weighted by $\boldsymbol{1}-\tilde{\boldsymbol{\varepsilon}}^L_k$. To avoid training instability caused by negative coefficients, we imposes an upper bound constraint of $1$ on the coefficients. This design ensures the monotonic decreasing property of the loss weighting coefficients similar to the effect of a learning rate decay mechanism, while dynamically adjusting the contribution of samples: gradually reducing the gradient update weights for samples that have been sufficiently fitted, while emphasizing attention on samples with significant residuals that have not yet been effectively represented by the model. 

\textbf{MU} For the weighting coefficients $\boldsymbol{\varepsilon}^U_k$ in unlearning tasks, the conditions are opposite to those of $\boldsymbol{\varepsilon}^L_k$: \ding{172} For the initial model, $\boldsymbol{\varepsilon}^U_0=\boldsymbol{1}$. \ding{173} For the final optimal model, $\boldsymbol{\varepsilon}^U_*=\boldsymbol{0}$. \ding{174} Similar to $\boldsymbol{\varepsilon}^L_k$, assuming homogeneity among samples, if $\ell(\theta ; z_i^U)>\ell(\theta ; z_j^U)$, then $\varepsilon_{k,i}<\varepsilon_{k,j}$. Similarly, based on the continuity in the function space, our heuristic estimation for the weighting coefficients in unlearning tasks includes a decrement with iteration steps and sample-wise adaptation based on loss magnitude.
\begin{equation}\label{eq: mu coefficient}
\tilde\varepsilon^U_{k,i}=(1-\frac{k}{K})\frac{1/\mathcal [\ell(\theta_k;z_i^U)]_{\text{detach}}^{\lambda^U}}{\sum_{z^U_j\in\mathcal D^U}1/[\ell(\theta_k;z^U_j)]_{\text{detach}}^{\lambda^U}}\times N^U,
\end{equation}
where $1\leq i \leq N^U$, $\lambda^U$ is the temperature scalar. $\tilde{\boldsymbol{\varepsilon}}^U_{k}$ is employed to modulate the gradient ascent loss for forgetting samples, thereby preventing model explosion and reducing the contribution to updates from samples whose influence has already been ablated, while prioritizing those whose losses remain minimal and are not yet adequately forgotten.

However, relying solely on empirical loss as an evaluation metric for sample contribution is limited and potentially biased. We believe that enhanced designs for coefficient estimation in future research could yield more accurate learning and unlearning results.

\subsection{Balanced Weight Saliency Mask}\label{subsec: saliency mask}

Recalling the saliency component in \eqref{eq: clu remain} of the theoretical framework, the optimization gradient is adjusted by the additional task-remain balancing Hessian $H^{\mathcal Q}_k(H_*^R)^{-1}$. However, online computation of the Hessian is extremely resource-intensive, and calculating the Hessian at the optimal model $\theta_*$ is also impractical. Therefore, an efficient approximation of both Hessians is necessary. For the latter, considering the remaining data, $H_*^R\approx H_0^R$. To align with our theoretical insight, we approximate the Hessian using the absolute value of the gradient and apply a hard thresholding operation, obtaining the weight saliency mask:
\begin{equation}\label{eq: saliency mask}
\mathbf{m}=\mathbf{I}\left[G_k^{\mathcal Q}(G^R_0)^{-1} \geq \gamma\right],
\end{equation}
where $G^{\mathcal Q}_k=|\nabla\mathcal L^{\mathcal Q}(\theta_k)|,G^R_0=|\nabla\mathcal L^R(\theta_0)|$, $\mathbf I[\cdot]$ is an element-wise indicator function, and $\gamma > 0$ is a hard threshold used to control the task-remain balance in selecting parameters. Previous CL and MU methods typically only consider task-related parameter saliency~\cite{chen2024mofomomentumfilteredoptimizermitigating,Fan2023SalUnEM}, whereas our balanced weight saliency mask, similar to balanced saliency methods\cite{Huang2024LearnFD,Foster2023FastMU}, simultaneously accounts for parameter saliency under the task-remain trade-off.

The saliency mask can enhance the task optimization process by directing updates to focus on the parameters that are crucial for knowledge updates. More advanced weight saliency estimation is expected to improve outcomes in future work.

\subsection{Integrated Fast-slow Weight Update}

By integrating the design of the four implementation components mentioned above with the task iteration formula from the theoretical framework, we obtain the proposed \textbf{U}nified \textbf{G}radient-based \textbf{C}ontinual \textbf{L}earning-\textbf{U}nlearning (\textbf{UG-CLU}). Specifically, in the inner loop for fast weights, we use adaptive coefficients from \eqref{eq: cl coefficient} and \eqref{eq: mu coefficient} to weight the task gradient update with the weight saliency mask from \eqref{eq: saliency mask}. Slow weights in the outer loop are updated by linearly interpolating the fine-tuned $\theta_k^R$ and $\theta_k$ in the weight space, achieving an estimated steepest descent for approximate CLU under the remaining output constraint in \eqref{eq: clu remain}. Thus, we derive the overall fast-slow weight update rule:
\begin{equation}\label{eq: inner loop}
\begin{aligned}
\text{Inner Loop : }&\mathop{\min}_{\theta^{\mathcal Q}_{k}}\mathcal L^R(\theta_k^R)\\
\text{s.t. }\theta_k^{\mathcal Q}=\theta_k-\beta_k &[\mathbf m\odot(-\nabla\mathcal L^{\mathcal Q}(\theta_k;\tilde{\boldsymbol \varepsilon}_{k}))],
\end{aligned}
\end{equation}
\begin{equation}\label{eq: outer loop}
\begin{aligned}
\text{Outer Loop : }\theta_{k+1}=&\,\theta_k-\alpha_k(\theta_k-\theta_k^R)\\\approx\theta_t-\alpha_k\beta_k^2\underbrace{(H_k^R)^{-1}}_{(R)}&[\underbrace{\mathbf m}_{(S)}\odot\underbrace{\nabla\mathcal L^{\mathcal Q}(\theta_k;\tilde{\boldsymbol \varepsilon}^{\mathcal Q}_{k})}_{(L)\&(U)}],
\end{aligned}
\end{equation}
where $\alpha_k$ denotes the learning rate of the slow weights. Consequently, we have derived an implementation from our theoretical results that can encompass both CL and MU within a unified framework. Even when switching between learning and unlearning tasks, our UG-CLU method only requires altering the computation of sample loss weighting, significantly enhancing system simplicity and component reusability, while effectively reducing the complexity of engineering implementation and maintenance. The complete algorithm is detailed in \textbf{Appendix.\,B}.
\section{Experiment}

\begin{figure}[tb]
    \centering
    \includegraphics[width=0.49\textwidth]{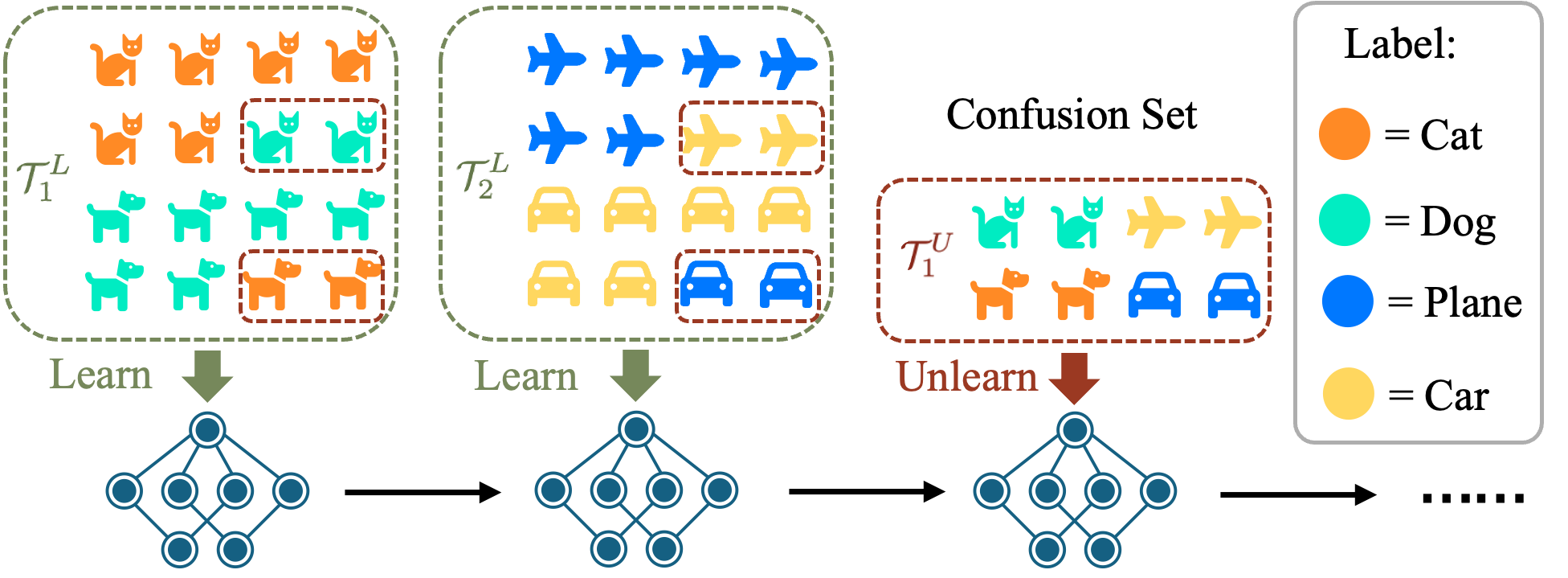}
    \caption{We adapt the interclass confusion unlearning setup from \cite{Goel2022TowardsAE} and extend it to the task-agnostic CLU system. Shapes and colors represent the actual and labeled classes, respectively. In each learning task, a portion of the data labels is shuffled and misclassified as noisy data, which needs to be addressed in the unlearning task to eliminate the harmful effects of these confusion sets on the model.}
    \label{fig: confusion}
\end{figure}

\subsection{CLU Settings and Evaluation}\label{subsec: setting and evaluation}

In terms of experimental setup, we extend from the class-incremental learning to task-aware and task-agnostic CLU setups, and introduce evaluation metrics for each setting.

\textbf{Class-Incremental Learning}~\cite{Hsu2018ReevaluatingCL,vandeVen2019ThreeSF}: For CL, we adopt the typical class-incremental learning setup, partitioning the dataset's class space into mutually exclusive and uniform subsets. The model gradually expands its classifier through sequential task training and is required to generate predictive responses for all seen classes during inference. \textbf{Evaluation} includes two core metrics: the average class accuracy of the final task, reflecting the model's learning accuracy (\textbf{LA}), and the measure of knowledge forgetting (\textbf{FM}), also known as Backward Transfer.

\textbf{Task-Aware CLU}: In previous literature~\cite{liu2022continual}, the class-incremental learning setup has been preliminarily extended to task-aware CLU, where requests to unlearn entire previously encountered tasks are inserted between training tasks. We use $(\mathcal Q_t, c_{t,1}, ..., c_{t,M_t})$ to denote the $t$-th task, with $+$ and $-$ as shorthand for learning and unlearning tasks, respectively, and $c_{t,m_t}$ representing the $M_t$ classes involved in the current task. \textbf{Evaluation}: In addition to the two evaluation metrics for CL methods, we also introduce MU evaluation metrics to assess the system's unlearning efficacy, including the average classification accuracy of the unlearned classes (\textbf{UA}) and the privacy leakage risk measured by membership inference attacks (\textbf{MIA})~\cite{Song2020SystematicEO,carlini2022membership,Jia2023ModelSC}.

\textbf{Task-Agnostic CLU}: We innovatively propose task-agnostic CLU setups that request fine-grained unlearning operations, including class-wise unlearning and random subset unlearning.
\begin{itemize}
    \item[$\bullet$] For \textbf{class-wise unlearning}, building upon the task-aware CLU system, we break the constraints of task boundaries, allowing selective unlearning of specific classes from historical learning tasks. \textbf{Evaluation} is the same as for task-aware CLU.
    \item[$\bullet$] For \textbf{random subset unlearning}, we adopt the interclass confusion setup proposed in \cite{Goel2022TowardsAE}. Specifically, based on the class-incremental learning setup, a subset of data is extracted from the training data of each task, and their labels are randomly shuffled so that they no longer belong to their original classes, forming a confusion set. These samples become noisy data in the learning task, and in subsequent unlearning tasks, the model is required to eliminate the impact of noise without knowing the correct labels, as illustrated in \textbf{Figure.\,\ref{fig: confusion}}. \textbf{Evaluation}: For the CL aspect, we still use \textbf{LA} and \textbf{FM}. For the MU aspect, we use unlearning accuracy (\textbf{UA}), i.e., the accuracy with which the confusion set is classified into labeled classes (the lower, the better), and the clean accuracy (\textbf{CA}), i.e., the accuracy with which the confusion set is classified into their original correct labels (the higher, the better). These metrics assess whether the unlearned model can generalize correctly on these data while discarding the unstable MIA on random subset unlearning.
\end{itemize}

\textbf{General Evaluation}: In addition to task-specific metrics, we also compute the output KL divergence metric naturally derived from approximate CLU (\textbf{KL}), as well as run-time efficiency (\textbf{RTE}). Detailed calculations for all evaluation metrics can be found in the \textbf{Appendix.\,C}.

\begin{table}[tb]
  \centering
  \caption{Performance summary of \textbf{task-aware CLU} on CIFAR-10 using ResNet-18. This table includes evaluations of Joint and ER-FT baselines with 5 MU methods paired with optimal CL methods combinations, CLU approaches CLPU-DER++ and UniCLUN, and our proposed UG-CLU. The result format is given by $a_{\pm b}$ with mean $a$ and standard deviation $b$ over 5 independent trials. Top-performing values are \textbf{boldfaced} with secondary leaders \underline{underlined}. RTE is reported in minutes.}
  \label{tab: cifar10-unlearn-task}
  \resizebox{0.49\textwidth}{!}{
  \begin{tabular}{cccccccc}
    \toprule
\multicolumn{2}{c}{\textbf{Method}} 
& \multicolumn{6}{c}{\textbf{(+0,1),(+2,3),(-0,1),(+4,5),(+6,7),(-4,5),(+8,9)}}  \\

\cmidrule(r){3-8} 

\textbf{CL} & \textbf{MU} & \textbf{LA} $\uparrow$ & \textbf{FM} $\uparrow$ & \textbf{UA} $\downarrow$ & \textbf{MIA} $\downarrow$ & \textbf{KL} $\downarrow$ & \textbf{RTE} \\

\midrule

\multicolumn{2}{c}{Joint-FT}
& 95.77\textsubscript{\textpm0.0}
& -
& \textbf{0.00}\textsubscript{\textpm0.0}
& \textbf{0.00}\textsubscript{\textpm0.0}
& 0.2196\textsubscript{\textpm0.0}
& -

\\

ER & FT
& 87.93\textsubscript{\textpm0.3}
& -3.97\textsubscript{\textpm0.5}
& 10.04\textsubscript{\textpm0.2}
& 12.50\textsubscript{\textpm0.2}
& 0.8605\textsubscript{\textpm0.1}
& 39.29

\\

\midrule

ER-EWC & GA
& 89.43\textsubscript{\textpm0.3}
& -3.62\textsubscript{\textpm0.3}
& \textbf{0.00}\textsubscript{\textpm0.0}
& 0.14\textsubscript{\textpm0.0}
& 0.8501\textsubscript{\textpm0.1}
& 43.76

\\

ER-ACE & NegGrad+
& 89.68\textsubscript{\textpm0.3}
& -6.50\textsubscript{\textpm0.1}
& \textbf{0.00}\textsubscript{\textpm0.0}
& 0.03\textsubscript{\textpm0.0}
& 0.8757\textsubscript{\textpm0.1}
& 40.57

\\

ER-ACE & SalUn
& 90.05\textsubscript{\textpm0.4}
& -6.25\textsubscript{\textpm0.2}
& \textbf{0.00}\textsubscript{\textpm0.0}
& \textbf{0.00}\textsubscript{\textpm0.0}
& 0.8320\textsubscript{\textpm0.1}
& 41.55

\\

DER++ & SCRUB
& 89.31\textsubscript{\textpm0.3}
& -3.33\textsubscript{\textpm0.3}
& \textbf{0.00}\textsubscript{\textpm0.0}
& 6.40\textsubscript{\textpm0.5}
& 0.9007\textsubscript{\textpm0.2}
& 55.31

\\

ER-ACE & L2UL
& 86,43\textsubscript{\textpm0.3}
& \underline{-2.37}\textsubscript{\textpm0.1}
& \textbf{0.00}\textsubscript{\textpm0.0}
& \textbf{0.00}\textsubscript{\textpm0.0}
& 1.2310\textsubscript{\textpm0.1}
& 60.89

\\

\midrule
\multicolumn{2}{c}{CLPU-DER++}
& \underline{90.75}\textsubscript{\textpm0.5}
& \textbf{-2.25}\textsubscript{\textpm0.2}
& \textbf{0.00}\textsubscript{\textpm0.0}
& \textbf{0.00}\textsubscript{\textpm0.0}
& \underline{0.8152}\textsubscript{\textpm0.4}
& 65.23

\\
\multicolumn{2}{c}{UniCLUN}
& 87.28\textsubscript{\textpm0.3}
& -5.00\textsubscript{\textpm0.5}
& \textbf{0.00}\textsubscript{\textpm0.0}
& \textbf{0.00}\textsubscript{\textpm0.0}
& 0.9972\textsubscript{\textpm0.1}
& 42.50

\\
\rowcolor{gray!20}
\multicolumn{2}{c}{UG-CLU}
& \textbf{91.72}\textsubscript{\textpm0.2}
& -6.10\textsubscript{\textpm0.1}
& \textbf{0.00}\textsubscript{\textpm0.0}
& \textbf{0.00}\textsubscript{\textpm0.0}
& \textbf{0.7084}\textsubscript{\textpm0.1}
& 39.19

\\ 
\bottomrule
\end{tabular}}
\end{table}

\subsection{Baselines and Implementation}

We select multiple baselines to compare and validate the effectiveness of our proposed method in CLU. First, we established the \textbf{Joint-RT} model as the theoretical performance upper bound for approximate CLU. Next, we included methods specifically designed for CLU, such as \textbf{CLPU-DER++}~\cite{liu2022continual} for task-aware scenarios and \textbf{UniCLUN}~\cite{chatterjee2024unifiedframeworkcontinuallearning}, which can be applied to task-agnostic CLU. Finally, we integrate combinations of CL and MU baselines. The CL component encompasses three primary gradient-based techniques: regularization methods (\textbf{EWC}~\cite{Kirkpatrick2016OvercomingCF}, \textbf{LwF}~\cite{Li2016LearningWF}), replay methods (\textbf{ER}~\cite{Chaudhry2019ContinualLW}, \textbf{ER-ACE}~\cite{Caccia2021NewIO}, \textbf{DER++}~\cite{Buzzega2020DarkEF}), and parameter selection methods (\textbf{MoFO}~\cite{chen2024mofomomentumfilteredoptimizermitigating}), all equipped with the ER standard memory buffer to ensure comparability. The MU component includes six strategies: fine-tuning only on retained data (\textbf{FT}~\cite{Warnecke2021MachineUO}), gradient ascent only on unlearning data (\textbf{GA}~\cite{Graves2020AmnesiacML,Thudi2021UnrollingSU}) with its improved version NegGrad+~\cite{Kurmanji2023MachineUI}, \textbf{SalUn}~\cite{Fan2023SalUnEM} using random labels with weight saliency for unlearning, \textbf{SCRUB}~\cite{Kurmanji2023TowardsUM} employing a bad teacher for two-stage unlearning via knowledge distillation, and \textbf{L2UL}~\cite{Cha2023LearningTU} using adversarial samples to guide unlearning and parameter regularization to maintain performance. The basic combination used ER and FT to construct the minimum system, while for the remaining five MU methods (GA/NegGrad+/SalUn/SCRUB/L2UL), a grid search is conducted across five CL methods (EWC/LwF/ER-ACE/DER++/MoFO) to select the combination that best approximates the KL divergence of the oracle model.

The experiments include two typical scenarios: CIFAR-10~\cite{Krizhevsky2009LearningML} using ResNet-18~\cite{He2015DeepRL}, and TinyImageNet~\cite{Le2015TinyIV} using Swin-T~\cite{Liu2021SwinTH}. The buffer size is 5,000 samples. Other hyperparameters are detailed in \textbf{Appendix\,D}.

\begin{table*}[ht]
  \centering
  \caption{Performance summary of \textbf{task-agnostic CLU} on CIFAR-10 with 2 task sequential orders using ResNet-18. This table includes evaluations of Joint and ER-FT baselines with 5 MU methods paired with optimal CL methods combinations, CLU approach UniCLUN, and our proposed UG-CLU. The result format is given by $a_{\pm b}$ with mean $a$ and standard deviation $b$ over 5 independent trials. Top-performing values are \textbf{boldfaced} with secondary leaders \underline{underlined}. RTE is reported in minutes.}
  \label{tab: cifar10-unlearn-class}
  \resizebox{0.95\textwidth}{!}{
  \begin{tabular}{cccccccccccccc}
    \toprule
\multicolumn{2}{c}{\textbf{Method}} 
& \multicolumn{6}{c}{\textbf{(+0,1),(+2,3),(-0),(+4,5),(+6,7),(-5),(+8,9),(-3)}} 
& \multicolumn{6}{c}{\textbf{(+0,1),(+2,3),(+4,5),(-1),(+6,7),(-2,3),(+8,9),(-4,7)}} \\

\cmidrule(r){3-8}  \cmidrule(r){9-14}

\textbf{CL} & \textbf{MU} & \textbf{LA} $\uparrow$ & \textbf{FM} $\uparrow$ & \textbf{UA} $\downarrow$ & \textbf{MIA} $\downarrow$ & \textbf{KL} $\downarrow$ & \textbf{RTE} 
& \textbf{LA} $\uparrow$ & \textbf{FM} $\uparrow$ & \textbf{UA} $\downarrow$ & \textbf{MIA} $\downarrow$ & \textbf{KL} $\downarrow$ & \textbf{RTE} \\

\midrule

\multicolumn{2}{c}{Joint-FT}
& 96.41\textsubscript{\textpm0.0}
& -
& 0.00\textsubscript{\textpm0.0}
& 0.00\textsubscript{\textpm0.0}
& 0.2366\textsubscript{\textpm0.0}
& -

& 96.96\textsubscript{\textpm0.0}
& -
& 0.00\textsubscript{\textpm0.0}
& 0.00\textsubscript{\textpm0.0}
& 0.1858\textsubscript{\textpm0.0}
& -

\\

ER & FT
& 89.01\textsubscript{\textpm0.4}
& -2.80\textsubscript{\textpm0.1}
& 20.20\textsubscript{\textpm0.2}
& 27.97\textsubscript{\textpm0.2}
& 1.7203\textsubscript{\textpm0.1}
& 41.41

& 92.64\textsubscript{\textpm0.3}
& -1.46\textsubscript{\textpm0.1}
& 17.26\textsubscript{\textpm0.2}
& 44.77\textsubscript{\textpm0.2}
& 2.1298\textsubscript{\textpm0.1}
& 43.69

\\

\midrule

ER-ACE & GA
& 88.38\textsubscript{\textpm0.5}
& -4.16\textsubscript{\textpm0.1}
& 2.40\textsubscript{\textpm0.2}
& 19.75\textsubscript{\textpm0.6}
& 1.4293\textsubscript{\textpm0.8}
& 42.47

& 85.18\textsubscript{\textpm0.1}
& -3.10\textsubscript{\textpm0.0}
& 7.92\textsubscript{\textpm0.0}
& 7.46\textsubscript{\textpm0.0}
& 2.3113\textsubscript{\textpm0.3}
& 42.83

\\

ER-ACE & NegGrad+
& 89.05\textsubscript{\textpm0.4}
& \underline{-3.59}\textsubscript{\textpm0.1}
& 1.67\textsubscript{\textpm0.1}
& 4.15\textsubscript{\textpm0.1}
& 1.2660\textsubscript{\textpm0.0}
& 43.29

& \underline{92.26}\textsubscript{\textpm0.0}
& \underline{-3.04}\textsubscript{\textpm0.1}
& \textbf{0.00}\textsubscript{\textpm0.0}
& 0.05\textsubscript{\textpm0.0}
& 1.5522\textsubscript{\textpm0.1}
& 43.45

\\

DER++ & SalUn
& 88.38\textsubscript{\textpm0.5}
& -4.47\textsubscript{\textpm0.1}
& \textbf{0.00}\textsubscript{\textpm0.0}
& \textbf{0.00}\textsubscript{\textpm0.0}
& 1.5800\textsubscript{\textpm0.1}
& 58.17

& 90.18\textsubscript{\textpm0.1}
& -3.34\textsubscript{\textpm0.0}
& \textbf{0.00}\textsubscript{\textpm0.0}
& \textbf{0.00}\textsubscript{\textpm0.0}
& 1.9552\textsubscript{\textpm0.2}
& 58.64

\\

DER++ & SCRUB
& \underline{89.10}\textsubscript{\textpm0.1}
& -3.87\textsubscript{\textpm0.1}
& 0.40\textsubscript{\textpm0.2}
& 7.39\textsubscript{\textpm0.3}
& \underline{1.1962}\textsubscript{\textpm0.1}
& 57.70

& 90.36\textsubscript{\textpm0.1}
& -4.00\textsubscript{\textpm0.0}
& 0.02\textsubscript{\textpm0.0}
& 0.02\textsubscript{\textpm0.0}
& \underline{1.2726}\textsubscript{\textpm0.3}
& 58.75

\\

ER-EWC & L2UL
& 88.77\textsubscript{\textpm0.1}
& -3.77\textsubscript{\textpm0.1}
& 16.23\textsubscript{\textpm0.3}
& 13.97\textsubscript{\textpm0.3}
& 1.7492\textsubscript{\textpm0.1}
& 64.28

& 84.04\textsubscript{\textpm0.1}
& -8.30\textsubscript{\textpm0.1}
& 7.36\textsubscript{\textpm0.1}
& 7.61\textsubscript{\textpm0.2}
& 2.0940\textsubscript{\textpm0.1}
& 73.58

\\

\midrule

\multicolumn{2}{c}{UniCLUN}
& 86.53\textsubscript{\textpm0.1}
& -6.19\textsubscript{\textpm0.1}
& \textbf{0.00}\textsubscript{\textpm0.0}
& \textbf{0.00}\textsubscript{\textpm0.0}
& 1.3605\textsubscript{\textpm0.1}
& 47.20

& 89.30\textsubscript{\textpm0.2}
& -5.04\textsubscript{\textpm0.2}
& \textbf{0.00}\textsubscript{\textpm0.0}
& \textbf{0.00}\textsubscript{\textpm0.0}
& 1.7872\textsubscript{\textpm0.1}
& 43.69

\\
\rowcolor{gray!20}
\multicolumn{2}{c}{UG-CLU}
& \textbf{92.33}\textsubscript{\textpm0.1}
& \textbf{-3.11}\textsubscript{\textpm0.1}
& \textbf{0.00}\textsubscript{\textpm0.0}
& \textbf{0.00}\textsubscript{\textpm0.0}
& \textbf{1.0247}\textsubscript{\textpm0.1}
& 41.34

& \textbf{92.86}\textsubscript{\textpm0.2}
& \textbf{-3.02}\textsubscript{\textpm0.2}
& \textbf{0.00}\textsubscript{\textpm0.0}
& \textbf{0.00}\textsubscript{\textpm0.0}
& \textbf{1.1282}\textsubscript{\textpm0.1}
& 41.44

\\ 
\bottomrule
\end{tabular}}
\end{table*}

\begin{table*}[ht]
  \centering
  \caption{Performance summary of \textbf{task-agnostic CLU} on TinyImageNet with 2 task sequential orders using Swin-T. This table includes evaluations of Joint and ER-FT baselines with 5 MU methods paired with optimal CL methods combinations, CLU approach UniCLUN, and our proposed UG-CLU. The result format is given by $a_{\pm b}$ with mean $a$ and standard deviation $b$ over 5 independent trials. Top-performing values are \textbf{boldfaced} with secondary leaders \underline{underlined}. RTE is reported in minutes. The task sequential order 1 is '(+0\dots19), (+20\dots39), (-20), (+40\dots59), (+60\dots79), (-60), (+80\dots99), (+100\dots119), (-100), (+120\dots139), (+140\dots159), (-140), (+160\dots179), (+180\dots199), (-180)'. The order 2 is '(+0\dots19), (+20\dots39), (+40\dots59), (+60\dots79), (-2,23), (+80\dots99), (+100\dots119),(+120\dots139), (-7,34,58,83), (+140\dots159), (+160\dots179), (+180\dots199), (-17,48,101,138)'.}
  \label{tab: tinyimagenet-unlearn-class}
  \resizebox{0.95\textwidth}{!}{
  \begin{tabular}{cccccccccccccc}
    \toprule
\multicolumn{2}{c}{\textbf{Method}} 
& \multicolumn{6}{c}{\textbf{Order 1}} 
& \multicolumn{6}{c}{\textbf{Order 2}} \\

\cmidrule(r){3-8}  \cmidrule(r){9-14}

\textbf{CL} & \textbf{MU} & \textbf{LA} $\uparrow$ & \textbf{FM} $\uparrow$ & \textbf{UA} $\downarrow$ & \textbf{MIA} $\downarrow$ & \textbf{KL} $\downarrow$ & \textbf{RTE} 
& \textbf{LA} $\uparrow$ & \textbf{FM} $\uparrow$ & \textbf{UA} $\downarrow$ & \textbf{MIA} $\downarrow$ & \textbf{KL} $\downarrow$ & \textbf{RTE} \\

\midrule

\multicolumn{2}{c}{Joint-FT}
& 85.48\textsubscript{\textpm0.1}
& -
& 0.00\textsubscript{\textpm0.0}
& 0.00\textsubscript{\textpm0.0}
& 0.0580\textsubscript{\textpm0.0}
& -

& 85.39\textsubscript{\textpm0.0}
& -
& 0.00\textsubscript{\textpm0.0}
& 0.00\textsubscript{\textpm0.0}
& 0.0760\textsubscript{\textpm0.0}
& -

\\

ER & FT
& 72.55\textsubscript{\textpm0.3}
& -18.67\textsubscript{\textpm0.5}
& 21.6\textsubscript{\textpm0.2}
& 99.92\textsubscript{\textpm0.0}
& 1.8205\textsubscript{\textpm0.4}
& 105.46

& 72.42\textsubscript{\textpm0.1}
& -18.23\textsubscript{\textpm0.1}
& 40.80\textsubscript{\textpm0.2}
& 84.92\textsubscript{\textpm0.2}
& 2.0839\textsubscript{\textpm0.1}
& 102.87

\\

\midrule

ER-ACE & GA
& 74.03\textsubscript{\textpm0.7}
& -13.49\textsubscript{\textpm0.7}
& 12.00\textsubscript{\textpm0.4}
& 31.04\textsubscript{\textpm0.3}
& 1.4468\textsubscript{\textpm0.2}
& 205.45

& 69.78\textsubscript{\textpm1.9}
& -17.66\textsubscript{\textpm1.3}
& 12.80\textsubscript{\textpm0.5}
& 35.02\textsubscript{\textpm0.6}
& 2.2960\textsubscript{\textpm0.2}
& 188.79

\\

DER++ & NegGrad+
& 72.67\textsubscript{\textpm0.6}
& -15.54\textsubscript{\textpm0.7}
& 3.20\textsubscript{\textpm0.1}
& 6.48\textsubscript{\textpm0.2}
& 1.7120\textsubscript{\textpm0.3}
& 347.34

& 73.57\textsubscript{\textpm0.4}
& -17.17\textsubscript{\textpm0.7}
& \textbf{2.40}\textsubscript{\textpm0.1}
& \underline{4.08}\textsubscript{\textpm0.2}
& 1.8192\textsubscript{\textpm0.3}
& 243.26

\\

DER++ & SalUn
& 74.21\textsubscript{\textpm0.5}
& -16.61\textsubscript{\textpm0.5}
& \textbf{0.00}\textsubscript{\textpm0.0}
& \underline{0.44}\textsubscript{\textpm0.0}
& 1.7773\textsubscript{\textpm0.1}
& 244.39

& \underline{74.61}\textsubscript{\textpm0.5}
& \underline{-14.16}\textsubscript{\textpm0.5}
& 22.40\textsubscript{\textpm0.2}
& 6.46\textsubscript{\textpm0.1}
& 2.2206\textsubscript{\textpm0.1}
& 213.63

\\

ER-LwF & SCRUB
& 74.03\textsubscript{\textpm0.2}
& -12.96\textsubscript{\textpm0.4}
& 12.80\textsubscript{\textpm0.3}
& 18.60\textsubscript{\textpm0.9}
& \underline{1.4397}\textsubscript{\textpm0.1}
& 196.71

& 73.36\textsubscript{\textpm0.4}
& -17.30\textsubscript{\textpm0.5}
& 11.60\textsubscript{\textpm0.2}
& 20.52\textsubscript{\textpm0.1}
& \underline{1.7606}\textsubscript{\textpm0.2}
& 168.91

\\

ER-MoFO & L2UL
& \underline{74.45}\textsubscript{\textpm0.6}
& \textbf{-12.35}\textsubscript{\textpm0.6}
& 24.80\textsubscript{\textpm0.2}
& 10.00\textsubscript{\textpm0.2}
& 1.9235\textsubscript{\textpm0.8}
& 567.52

& 73.51\textsubscript{\textpm0.5}
& -18.00\textsubscript{\textpm0.6}
& 10.80\textsubscript{\textpm0.2}
& 11.52\textsubscript{\textpm0.1}
& 2.4649\textsubscript{\textpm0.3}
& 233.23

\\

\midrule

\multicolumn{2}{c}{UniCLUN}
& 70.46\textsubscript{\textpm0.6}
& -23.07\textsubscript{\textpm0.3}
& 4.80\textsubscript{\textpm0.2}
& 17.24\textsubscript{\textpm0.4}
& 2.3542\textsubscript{\textpm0.2}
& 137.41

& 71.24\textsubscript{\textpm0.2}
& -19.72\textsubscript{\textpm0.3}
& 19.60\textsubscript{\textpm0.5}
& 18.78\textsubscript{\textpm0.4}
& 2.7132\textsubscript{\textpm0.4}
& 157.48

\\
\rowcolor{gray!20}
\multicolumn{2}{c}{UG-CLU}
& \textbf{75.27}\textsubscript{\textpm0.4}
& \underline{-12.68}\textsubscript{\textpm0.1}
& \textbf{0.00}\textsubscript{\textpm0.0}
& \textbf{0.00}\textsubscript{\textpm0.0}
& \textbf{1.3661}\textsubscript{\textpm0.3}
& 190.28

& \textbf{76.49}\textsubscript{\textpm0.3}
& \textbf{-13.74}\textsubscript{\textpm0.3}
& \textbf{2.40}\textsubscript{\textpm0.1}
& \textbf{1.00}\textsubscript{\textpm0.1}
& \textbf{1.5933}\textsubscript{\textpm0.1}
& 176.03

\\ 
\bottomrule
\end{tabular}}
\end{table*}

\subsection{Performance on Task-aware CLU}

The experimental results in \textbf{Table.\,\ref{tab: cifar10-unlearn-task}} indicate that in the task-aware setting, since the data distribution of unlearning tasks is entirely consistent with historical learning tasks and the task sequence structure is simple, most MU methods can achieve ideal unlearning efficacy that UA drops to 0, and the success rate of MIA is 0, fully meeting privacy protection requirements. It is worth noting that DER++, which typically performs state-of-the-art in CL, does not necessarily remain optimal when combined with MU methods. ER-ACE demonstrates the best synergistic effect and maintains stable knowledge retention capabilities when combined with various MU methods, highlighting the matching mechanism between CL and MU methods in CLU systems. Notably, while the specialized method CLPU-DER++ achieves good unlearning results, its reliance on task-level model snapshots significantly increases computational overhead. In contrast, the proposed UG-CLU method not only ensures complete unlearning efficacy but also achieves the highest learning accuracy, with the closest KL divergence to the oracle model's output distribution, while maintaining high computational efficiency.

\subsection{Performance on Task-agnostic CLU for class-wise unlearning}

Under the task-agnostic CLU framework proposed in this paper, for complex scenarios involving cross-task class unlearning (including both intra-task and cross-task unlearning sequences), the experimental results in \textbf{Tables.\,\ref{tab: cifar10-unlearn-class},\ref{tab: tinyimagenet-unlearn-class}} demonstrate that traditional methods such as GA and L2UL struggle to balance unlearning efficacy and model performance when faced with cross-interference from knowledge iterations. Among the six MU methods, only SalUn achieves stable unlearning across datasets, but its random label strategy causes the output distribution to significantly deviate from the oracle model. A comparison of CL methods reveals that the soft-label replay mechanism DER++ exhibits the best knowledge consolidation capability. Our UG-CLU method demonstrates adaptability in complex knowledge management scenarios, effectively unlearning across all scenarios while maintaining model generalization ability. The most direct evidence is that our model consistently achieves the lowest KL divergence from the oracle model's output distribution, successfully balancing privacy protection, model performance, and computational cost.

\begin{table}[tb]
  \centering
  \caption{Performance summary of \textbf{task-agnostic CLU} on CIFAR-10 with interclass confusion using ResNet-18. This table includes evaluations of Joint and ER-FT baselines with 5 MU methods paired with optimal CL methods combinations, CLU approach UniCLUN, and our proposed UG-CLU. The result format is given by $a_{\pm b}$ with mean $a$ and standard deviation $b$ over 5 independent trials. Top-performing values are \textbf{boldfaced} with secondary leaders \underline{underlined}. RTE is reported in minutes.}
  \label{tab: cifar10-unlearn-confusion}
  \resizebox{0.49\textwidth}{!}{
  \begin{tabular}{cccccccc}
    \toprule
\multicolumn{2}{c}{\textbf{Method}} 
& \multicolumn{6}{c}{\textbf{Interclass Confusion}}  \\

\cmidrule(r){3-8} 

\textbf{CL} & \textbf{MU} & \textbf{LA} $\uparrow$ & \textbf{FM} $\uparrow$ & \textbf{UA} $\downarrow$ & \textbf{CA} $\uparrow$ & \textbf{KL} $\downarrow$ & \textbf{RTE} \\

\midrule

\multicolumn{2}{c}{Joint-FT}
& 94.58\textsubscript{\textpm0.3}
& -
& 0.00\textsubscript{\textpm0.0}
& 95.2\textsubscript{\textpm0.1}
& 0.1825\textsubscript{\textpm0.0}
& -

\\

ER & FT
& 77.75\textsubscript{\textpm0.4}
& -10.46\textsubscript{\textpm0.1}
& 24.48\textsubscript{\textpm0.2}
& 36.96\textsubscript{\textpm0.1}
& 2.6246\textsubscript{\textpm0.2}
& 49.04

\\

\midrule

DER++ & GA
& 57.74\textsubscript{\textpm1.6}
& -28.35\textsubscript{\textpm1.1}
& 4.44\textsubscript{\textpm0.2}
& 56.00\textsubscript{\textpm1.2}
& 2.8801\textsubscript{\textpm0.6}
& 60.03

\\

ER-ACE & NegGrad+
& \underline{81.37}\textsubscript{\textpm0.2}
& -3.34\textsubscript{\textpm0.1}
& 3.92\textsubscript{\textpm0.1}
& \underline{72.40}\textsubscript{\textpm0.2}
& \underline{1.3721}\textsubscript{\textpm0.0}
& 46.14

\\

DER++ & SalUn
& 75.27\textsubscript{\textpm0.2}
& -1.96\textsubscript{\textpm0.1}
& \underline{2.72}\textsubscript{\textpm0.1}
& 65.48\textsubscript{\textpm0.3}
& 2.0577\textsubscript{\textpm0.1}
& 59.95

\\

ER-EWC & SCRUB
& 81.35\textsubscript{\textpm0.1}
& -4.18\textsubscript{\textpm0.0}
& 12.04\textsubscript{\textpm0.1}
& 70.32\textsubscript{\textpm0.2}
& 1.3755\textsubscript{\textpm0.0}
& 44.57

\\

ER-ACE & L2UL
& 77.37\textsubscript{\textpm0.1}
& \textbf{-1.37}\textsubscript{\textpm0.0}
& 9.44\textsubscript{\textpm0.1}
& 65.08\textsubscript{\textpm0.2}
& 1.7228\textsubscript{\textpm0.1}
& 45.39

\\

\midrule

\multicolumn{2}{c}{UniCLUN}
& 73.32\textsubscript{\textpm0.1}
& -14.3\textsubscript{\textpm0.2}
& 12.34\textsubscript{\textpm0.2}
& 65.68\textsubscript{\textpm0.2}
& 2.1391\textsubscript{\textpm0.1}
& 43.16

\\
\rowcolor{gray!20}
\multicolumn{2}{c}{UG-CLU}
& \textbf{82.42}\textsubscript{\textpm0.1}
& \underline{-1.45}\textsubscript{\textpm0.1}
& \textbf{2.20}\textsubscript{\textpm0.0}
& \textbf{78.36}\textsubscript{\textpm0.1}
& \textbf{1.2708}\textsubscript{\textpm0.1}
& 44.37

\\ 
\bottomrule
\end{tabular}}
\end{table}

\subsection{Performance on Task-agnostic CLU for random subset unlearning}

In the task-agnostic CLU scenario, we also design a random subset unlearning evaluation protocol for interclass confusion. This protocol requires the CLU agent to eliminate the memory of incorrect labels in the confusion set while maintaining the generalization and classification capabilities for the samples in the confusion set. As shown in the experimental results in \textbf{Table.\,\ref{tab: cifar10-unlearn-confusion}}, GA significantly degrades the overall model performance when unlearning the confusion set, while MU methods like SalUn and L2UL, which use incorrect labels for guidance, achieve suppression of the confusion set label accuracy. However, the injection of incorrect knowledge causes the samples to deviate from their original feature distribution, resulting in suboptimal performance on CA. Only unbiased guidance MU methods such as NegGrad+, SCRUB, and UniCLUN can balance the requirements of knowledge elimination and feature preservation. Our UG-CLU framework further incorporates manifold improvements for remaining data, achieving the best final LA and generalization performance on the confusion set. More importantly, the output distribution analysis shows that UG-CLU has the lowest KL divergence from the oracle model, confirming the precise control capability of our method during the knowledge elimination process.

\subsection{Ablation Studies}


We performe ablation experiments on the three key modules proposed in \textbf{Section.\,\ref{sec: method}}, with the results shown in \textbf{Table.\,\ref{tab: cifar10-ablation}}. The vanilla gradient descent iteration method for approximate CLU in our theory is the same as ER-NegGrad+. Starting from this, we further apply the improvements of our remain-preserved manifold. Incorporating implicit online Hessian approximation ($\S$.\,\ref{subsec: implicit hessian}) improve the model's final learning accuracy and achieve better preservation of the model's general performance, but interfering effective unlearning without help of adaptive coefficients and saliency mask. Introducing adaptive coefficients ($\S$.\,\ref{subsec: adaptive coef}) provide a right direction to unlearning and result in a slight boost to overall performance, and finally, adding the weight saliency mask ($\S$.\,\ref{subsec: saliency mask}) formed our complete UG-CLU method, achieving optimal CLU performance and the closest output distribution to the oracle model. This validates the effectiveness of these key modules in our design.

\begin{table}[tb]
  \centering
  \caption{Performance summary of ablation on components of UG-CLU. Evaluated on \textbf{task-agnostic CLU} on CIFAR-10 using ResNet-18. We start with the baseline ER-NegGrad+ and overlay the modules we designed. The result format is given by $a_{\pm b}$ with mean $a$ and standard deviation $b$ over 5 independent trials. RTE is reported in minutes.}
  \label{tab: cifar10-ablation}
  \resizebox{0.49\textwidth}{!}{
  \begin{tabular}{ccccccc}
    \toprule
\multirow{2}{*}{\textbf{Method}} 
& \multicolumn{6}{c}{\textbf{(+0,1),(+2,3),(-0),(+4,5),(+6,7),(-5),(+8,9),(-3)}}  \\

\cmidrule(r){2-7} 

 & \textbf{LA} $\uparrow$ & \textbf{FM} $\uparrow$ & \textbf{UA} $\downarrow$ & \textbf{MIA} $\downarrow$ & \textbf{KL} $\downarrow$ & \textbf{RTE} \\

\midrule

Joint-FT
& 96.41\textsubscript{\textpm0.0}
& -
& 0.00\textsubscript{\textpm0.0}
& 0.00\textsubscript{\textpm0.0}
& 0.2366\textsubscript{\textpm0.0}
& -

\\

ER-NegGrad+
& 85.38\textsubscript{\textpm0.5}
& -5.96\textsubscript{\textpm0.3}
& 0.27\textsubscript{\textpm0.0}
& 0.37\textsubscript{\textpm0.0}
& 1.6691\textsubscript{\textpm0.1}
& 43.41

\\
\midrule

+ \S.\,\ref{subsec: implicit hessian}
& 89.83\textsubscript{\textpm0.3}
& -5.97\textsubscript{\textpm0.3}
& 4.20\textsubscript{\textpm0.0}
& 5.61\textsubscript{\textpm0.5}
& 1.3412\textsubscript{\textpm0.2}
& 40.23

\\

+ \S.\,\ref{subsec: adaptive coef}
& 90.87\textsubscript{\textpm0.4}
& -4.87\textsubscript{\textpm0.2}
& 0.13\textsubscript{\textpm0.0}
& 0.07\textsubscript{\textpm0.0}
& 1.2605\textsubscript{\textpm0.1}
& 40.58

\\

\rowcolor{gray!20}
+ \S.\,\ref{subsec: saliency mask}
& 92.33\textsubscript{\textpm0.1}
& -3.11\textsubscript{\textpm0.1}
& 0.00\textsubscript{\textpm0.0}
& 0.00\textsubscript{\textpm0.0}
& 1.0247\textsubscript{\textpm0.1}
& 41.34

\\ 
\bottomrule
\end{tabular}}
\end{table}

\section{Conclusion and Limitation}

This study, through the systematic integration of CL and MU, has constructed a unified theoretical framework for CLU system, prompting the knowledge regulation capabilities of artificial intelligence systems in dynamic environments. By deriving a four-dimensional optimization direction that integrates knowledge acquisition, unlearning, consolidation of prior knowledge, and gradient saliency modulation from the KL divergence minimization objective, we have laid a mathematical foundation for CLU systems. Innovatively, we introduce a remain-preserved manifold constraint, leveraging the Hessian matrix of remaining data to modulate parameter update trajectories, theoretically validating its effectiveness in isolating the interference of new knowledge learning and target unlearning on existing representations, thereby mitigating catastrophic forgetting during knowledge updates. We design a fast-slow weight update technique to implicitly approximate Hessian modulation, combined with adaptive loss weighting coefficient and balanced weight saliency mask to form an efficient and scalable UG-CLU framework. Additionally, we break through the limitations of traditional task-aware unlearning, pioneering a task-agnostic CLU setup that supports cross-task category and random sample-level fine-grained operations. The effectiveness of the UG-CLU method is validated across multiple datasets and model architectures, providing a methodology system that combines theoretical rigor and practical feasibility for building lifelong learning systems. We look forward to advancing the application of CLU systems in real-world scenarios.

Despite the breakthroughs in CLU systems, this study still has limitations in theoretical completeness, system optimization, and application expansion: although we provide a general gradient update method for CLU systems, analyzing the convergence of gradient iterations and their impact on generalization in deep CLU systems remains challenging. Moreover, this paper employs only a basic reservoir sampling strategy for remaining data distribution, and future work could explore optimization mechanisms such as dynamic importance sampling. Additionally, experimental validation focuses on image classification tasks, and the applicability to natural language processing and multimodal scenarios remains to be verified. These research gaps provide directions for further refinement of CLU systems.



\bibliographystyle{IEEEtran}
\bibliography{ref.bib}

\appendices
\setcounter{section}{0}
\setcounter{figure}{0}
\makeatletter 
\renewcommand{\thefigure}{A\arabic{figure}}
\renewcommand{\thetable}{A\arabic{table}}
\makeatother
\setcounter{table}{0}
\setcounter{algorithm}{0}
\renewcommand{\thealgorithm}{A\arabic{algorithm}}
\setcounter{equation}{0}
\renewcommand{\theequation}{A\arabic{equation}}
\clearpage

\section{Detailed Proof}

\subsection{Proof of Steepest Descent in Section.\,3.4}

\begin{proof}
We form the following optimization problem for the steepest descent, which is to find the direction $\delta\theta=\theta_{k+1}-\theta_k$ that drives the objective function $F(\theta)$ descent fastest within a $\xi$-neighborhood of the current parameters $\theta_k$, and the $\xi$-neighborhood is rendered by the manifold metric $\rho(\cdot,\cdot)$:
\begin{equation}\label{eq: steepest descent rewrite (appendix)}
\delta \theta =\mathop{\arg\min}_{\delta \theta} F(\theta_k+\delta \theta)\quad \text{s.t.}\quad \rho(\theta_k,\theta_k+\delta \theta) \leq \xi.
\end{equation}
We introduce a Lagrange multiplier $\eta\geq0$ to construct the Lagrangian $\tilde{\mathcal L}$ for this optimization problem:
\begin{equation}
    \tilde{\mathcal L}(\delta \theta, \eta)=F(\theta_k+\delta \theta) + \eta(\rho(\theta_k,\theta_k+\delta \theta) - \xi).
\end{equation}
Using the Karush-Kuhn-Tucker (KKT) theorem, we can take the derivative of $\tilde{\mathcal L}$ $w.r.t.$ $\delta \theta$ and set it to zero:
\begin{equation}\label{eq: KKT (appendix)}
\begin{aligned}
    &\nabla_{\delta \theta}\tilde{\mathcal L}(\delta \theta, \eta)\\
    =&\nabla_{\delta \theta}F(\theta_k+\delta \theta) + \eta\nabla_{\delta \theta}\rho(\theta_k,\theta_k+\delta \theta)=0,
\end{aligned}
\end{equation}
where $\eta$ depends on $\xi$ and $\theta_k$ implicitly. We can rewrite \eqref{eq: KKT (appendix)} by variable substitution $\theta_{k+1}=\theta_k+\delta \theta$ and $\eta=1/{\alpha_k(\xi,\theta_k)}$.
\begin{equation}
\begin{aligned}
    &\nabla_{\theta_{k+1}}\tilde{\mathcal L}(\theta_{k+1}, \eta)\\
    =&\nabla_{\theta_{k+1}}F(\theta_{k+1}) + \frac{1}{\alpha_k(\xi,\theta_k)}\nabla_{\theta_{k+1}}\rho(\theta_{k},\theta_{k+1})
    =0.
\end{aligned}
\end{equation}
Thus, the original problem is transformed into an unconstrained optimization problem $w.r.t.$ $\theta_{k+1}$, where the neighborhood size is implicitly given by $\alpha_k$:
\begin{equation}
    \theta_{k+1} =\mathop{\arg\min}_{\theta_{k+1}} F(\theta_{k+1}) + \frac{1}{\alpha_k(\xi,\theta_k)}\rho(\theta_{k},\theta_{k+1}).
\end{equation}
\end{proof}

\subsection{Proof of \textbf{Proposition.\,1} in Section.\,4.2}

The optimization problem of approximate CLU is to find the steepest descent direction that minimizes the KL divergence with the optimal output within the vicinity of the current model $\theta_k$:
\begin{equation}\label{eq: optimization problem (appendix)}
\begin{aligned}
\theta_{k+1}
=&\mathop{\arg\min}\limits_{\theta_{k+1}} D_{\text{KL}}\left(p_z(\theta_*)||p_z(\theta_{k+1})\right)+\frac{1}{\alpha_k}\rho(\theta_k,\theta_{k+1})\\
=&\mathop{\arg\min}\limits_{\theta_{k+1}}
\underbrace{D_{\text{KL}}\left(p_{z^R}(\theta_*)||p_{z^R}(\theta_{k+1})\right)}_{(R(\theta_{k+1}))}p^R\\ 
&+\underbrace{D_{\text{KL}}\left(p_{z^L}(\theta_*)||p_{z^L}(\theta_{k+1})\right)}_{(L(\theta_{k+1}))}p^L\\
&+\underbrace{D_{\text{KL}}\left(p_{z^U}(\theta_*)||p_{z^U}(\theta_{k+1})\right)}_{(U(\theta_{k+1}))}p^U\\ 
&+\frac{1}{\alpha_k}\underbrace{\rho(\theta_k,\theta_{k+1})}_{(C(\theta_{k+1}))}, 
\end{aligned}
\end{equation}

\begin{prop}\label{prop: clu vanilla (appendix)}
Under the Euclidean manifold metric, $\rho(\theta_k,\theta_{k+1})=\frac12\lVert \theta_k-\theta_{k+1} \rVert^2$. Assuming that the current model satisfies \textbf{Assumption.\,4}. Let $H_*^R=\nabla^2\mathcal L^R(\theta_*)$ denote the Hessian of the oracle model on the remaining set and $H_{k}^L=\nabla^2\mathcal L^L(\theta_{k}), H_{k}^U=\nabla^2\mathcal L^U(\theta_{k})$ denote the Hessian of the current model on the request sets, respectively. Then, the steepest descent direction that minimizes \eqref{eq: optimization problem (appendix)} is approximately:
\begin{equation}\label{eq: clu vanilla (appendix)}
\begin{aligned}
    \theta_{k+1}-\theta_k:\approx-\alpha_k
    [
    &\underbrace{\nabla \mathcal L^R(\theta_k)}_{(R)}p^R\\
    +\underbrace{\frac12(H^L_kp^L + H^U_kp^U)(H_*^R)^{-1}}_{(S)}[&\underbrace{\nabla \mathcal L^L(\theta_k;\boldsymbol{1}-\boldsymbol{\varepsilon}^L_k)}_{(L)}] \\
    + \underbrace{\frac12(H^L_kp^L + H^U_kp^U)(H_*^R)^{-1}}_{(S)}[&\underbrace{\nabla \mathcal L^U(\theta_k;-\boldsymbol{\varepsilon}^U_k)}_{(U)}]
    ].
\end{aligned}
\end{equation}
\end{prop}

\begin{proof}
We can decompose the original optimization problem into four parts: $R(\theta_{k+1})$, $L(\theta_{k+1})$, $U(\theta_{k+1})$, and $C(\theta_{k+1})$.

(\textbf{Part \uppercase\expandafter{\romannumeral1}}) First, we solve the remaining part $R(\theta_{k+1})$ by taking the first-order approximation at $\theta_k$:
\begin{equation}\label{eq: solved R theta_k+1 in euclidean (appendix)}
\begin{aligned}
    &R(\theta_{k+1})\\
    =&R(\theta_k) + \nabla R(\theta_k)^\top(\theta_{k+1}-\theta_k) \\
    =&R(\theta_k) - \mathbb{E}_{p(z^R;\theta_*)}\left[\nabla\log p(z^R;\theta_{k})\right]^\top(\theta_{k+1}-\theta_k) \\
    =&R(\theta_k) +\left[\nabla\mathcal L^R(\theta_k)\right]^\top(\theta_{k+1}-\theta_k).
\end{aligned}
\end{equation}

(\textbf{Part \uppercase\expandafter{\romannumeral2}}) Next, we solve the learning part $L(\theta_{k+1})$:
\begin{align}\label{eq: unsolved L theta_k+1 in euclidean (appendix)}
    L(\theta_{k+1})=L(\theta_k) + \nabla L(\theta_k)^\top(\theta_{k+1}-\theta_k).
\end{align}
We denote $\nabla L(\theta_k) = -\mathbb{E}_{p(z^L;\theta_*)}\left[\nabla \log p(z^L;\theta_{k})\right]=G^L(\theta_k)$, and then expand $G^L(\theta_*)$ at $\theta_k$:
\begin{equation}\label{eq: unsolved diff L theta_k in euclidean (appendix)}
\begin{aligned}
    &G^L(\theta_*)\\
    =&G^L(\theta_k)-\nabla G^L(\theta_k)(\theta_k-\theta_*) \\
    \Rightarrow &G^L(\theta_k)\\
    =&G^L(\theta_*)+\nabla G^L(\theta_k)(\theta_k-\theta_*) \\
    =&-\mathbb{E}_{p_{z^L}(\theta_*)}\left[\nabla \log p_{z^L}(\theta_{*})\right]-\mathbb{E}_{p_{z^L}(\theta_*)}\left[\nabla^2 \log p_{z^L}(\theta_{k})\right]\Delta_k \\
    =&0+H^L_k\Delta_k,
\end{aligned}
\end{equation}
where $H^L_k =-\mathbb{E}_{p(z^L;\theta_*)}\left[\nabla^2 \log p(z^L;\theta_k)\right]$ is the Hessian $w.r.t.$ the current model $\theta_*$ at learning set, and $\Delta_k=\theta_k-\theta_*$ is the difference between the current model $\theta_k$ and the optimal model $\theta_*$.
We cannot directly obtain the parameter difference $\Delta_k$ and need to estimate it. Recalling the assumption that 
\begin{equation}\label{eq: theta k argmin lu (apendix)}
\theta_{k} =\mathop{\arg\min}\limits_\theta \mathcal{L}^R(\theta) + \mathcal L^{L}(\theta;\boldsymbol{\varepsilon}^{L}_{k})+ \mathcal L^{U}(\theta;\boldsymbol{\varepsilon}^{U}_{k})
\end{equation}
\begin{equation}\label{eq: theta 0 argmin lu (apendix)}
\theta_{0} =\mathop{\arg\min}\limits_\theta \mathcal{L}^R(\theta) + \mathcal L^{L}(\theta;\boldsymbol{0})+ \mathcal L^{U}(\theta;\boldsymbol{1})
\end{equation}
\begin{equation}\label{eq: theta * argmin lu (apendix)}
\theta_{*} =\mathop{\arg\min}\limits_\theta \mathcal{L}^R(\theta) + \mathcal L^{L}(\theta;\boldsymbol{1})+ \mathcal L^{U}(\theta;\boldsymbol{0})
\end{equation}
We can utilize the optimality of $\theta_k$ on the weighted function to take the derivative $w.r.t.$ $\theta_k$ and set it to zero:
\begin{equation}
\begin{aligned}
    0=&\nabla \mathcal{L}^R(\theta_k) + \nabla \mathcal L^{L}(\theta_k;\boldsymbol{\varepsilon}^{L}_{k})+ \nabla \mathcal L^{U}(\theta_k;\boldsymbol{\varepsilon}^{U}_{k}) \\
    =&\left[\nabla \mathcal L^R(\theta_*) + \nabla^2 L^R(\theta_*)\Delta_k+o(\Delta_k)\right]\\
    &+ \nabla \mathcal L^{L}(\theta_k;\boldsymbol{\varepsilon}^{L}_{k}) 
    + \nabla \mathcal L^{U}(\theta_k;\boldsymbol{\varepsilon}^{U}_{k}) \\
    =& - \nabla \mathcal L^{L}(\theta_*;\boldsymbol{1}) + \nabla^2 \mathcal L^R(\theta_*)\Delta_k \\
    & + \nabla \mathcal L^{L}(\theta_k;\boldsymbol{\varepsilon}^{L}_{k}) 
    + \nabla \mathcal L^{U}(\theta_k;\boldsymbol{\varepsilon}^{U}_{k})\\
    = & - \left[\nabla \mathcal L^{L}(\theta_k;\boldsymbol{1}) - \nabla^2 \mathcal L^R(\theta_*)\Delta_k + o(\Delta_k) \right] \\
    & + \nabla^2 \mathcal L^R(\theta_*)\Delta_k +o(\Delta_k) \\
    & + \nabla \mathcal L^{L}(\theta_k;\boldsymbol{\varepsilon}^{L}_{k}) 
    + \nabla \mathcal L^{U}(\theta_k;\boldsymbol{\varepsilon}^{U}_{k})\\
    \approx& 2 \nabla^2 \mathcal L^R(\theta_*)\Delta_k
    - \nabla \mathcal L^{L}(\theta_k;\boldsymbol{1}-\boldsymbol{\varepsilon}^{L}_{k}) 
    + \nabla \mathcal L^{U}(\theta_k;\boldsymbol{\varepsilon}^{U}_{k})
    .
\end{aligned}
\end{equation}
Since $\theta_*$ minimizes $\mathcal{L}^R(\theta) + \mathcal L^{L}(\theta;\boldsymbol{1})$, $\nabla \mathcal{L}^R(\theta) + \mathcal L^{L}(\theta;\boldsymbol{1})=0$. By performing the Taylor expansion and dropping $o(\Delta_k)$ terms, we have
\begin{equation}
\begin{aligned}\label{eq: delta_k in euclidean (appendix)}
    &\Rightarrow \Delta_k  \\
    &\approx \frac12 \left[\nabla^2 \mathcal L^R(\theta_*)\right]^{-1}\left[\nabla \mathcal L^{L}(\theta_k;\boldsymbol{1}-\boldsymbol{\varepsilon}^{L}_{k}) 
    + \nabla \mathcal L^{U}(\theta_k;-\boldsymbol{\varepsilon}^{U}_{k})\right]\\
    &=\frac12\left(H^R_*\right)^{-1}\left[\nabla \mathcal L^{L}(\theta_k;\boldsymbol{1}-\boldsymbol{\varepsilon}^{L}_{k}) 
    + \nabla \mathcal L^{U}(\theta_k;-\boldsymbol{\varepsilon}^{U}_{k})\right].
\end{aligned}
\end{equation}
By plugging \eqref{eq: delta_k in euclidean (appendix)} into \eqref{eq: unsolved diff L theta_k in euclidean (appendix)}, we can get
\begin{equation}
\begin{aligned}\label{eq: solved diff L theta_k in euclidean (appendix)}
    &\nabla L(\theta_k)\\
    =\,& G^L(\theta_k) \\
    \approx\,& \frac12H^L_k\left(H^R_*\right)^{-1}\left[\nabla \mathcal L^{L}(\theta_k;\boldsymbol{1}-\boldsymbol{\varepsilon}^{L}_{k}) 
    + \nabla \mathcal L^{U}(\theta_k;-\boldsymbol{\varepsilon}^{U}_{k})\right].
\end{aligned}
\end{equation}

(\textbf{Part \uppercase\expandafter{\romannumeral3}}) We also solve the unlearning part $U(\theta_{k+1})$ with similar pipeline in solving $L(\theta_{k+1})$:
\begin{align}\label{eq: unsolved U theta_k+1 in euclidean (appendix)}
    U(\theta_{k+1})=U(\theta_k) + \nabla U(\theta_k)^\top(\theta_{k+1}-\theta_k).
\end{align}
We denote $\nabla U(\theta_k) = -\mathbb{E}_{p(z^U;\theta_*)}\left[\nabla \log p(z^U;\theta_{k})\right]=G^U(\theta_k)$, and then have:
\begin{equation}\label{eq: unsolved diff U theta_k in euclidean (appendix)}
\begin{aligned}
    G^U(\theta_*)=H^U_k\Delta_k,
\end{aligned}
\end{equation}
Bring in the result from \eqref{eq: delta_k in euclidean (appendix)}, we have:
\begin{equation}
\begin{aligned}\label{eq: solved diff U theta_k in euclidean (appendix)}
    &\nabla U(\theta_k)\\
    =\,& G^U(\theta_k) \\
    \approx\,& \frac12H^U_k\left(H^R_*\right)^{-1}\left[\nabla \mathcal L^{L}(\theta_k;\boldsymbol{1}-\boldsymbol{\varepsilon}^{L}_{k}) 
    + \nabla \mathcal L^{U}(\theta_k;-\boldsymbol{\varepsilon}^{U}_{k})\right].
\end{aligned}
\end{equation}

(\textbf{Part \uppercase\expandafter{\romannumeral4}}) Finally, we derive the constraint $C(\theta_{k+1})$ as follows,
\begin{align}\label{eq: solved C theta_k+1 in euclidean (appendix)}
    C(\theta_{k+1})
    =& C(\theta_k)+ \nabla C(\theta_k)^\top(\theta_{k+1}-\theta_k)\nonumber \\
    =& C(\theta_k)+ 2(\theta_{k+1}-\theta_k)^\top(\theta_{k+1}-\theta_k).
\end{align}
Substituting \eqref{eq: solved R theta_k+1 in euclidean (appendix)}, \eqref{eq: unsolved L theta_k+1 in euclidean (appendix)}, \eqref{eq: unsolved U theta_k+1 in euclidean (appendix)}, and \eqref{eq: solved C theta_k+1 in euclidean (appendix)} to each part in \eqref{eq: optimization problem (appendix)}, and take the derivative $w.r.t.$ $\theta_{k+1}$ of the minimization problem to derive the optimal solution, we have
\begin{equation}
\begin{aligned}
    0=&\nabla R(\theta_{k+1})p^R + \nabla L(\theta_{k+1})p^L \\
    &+ \nabla U(\theta_{k+1})p^U + \frac{1}{2\alpha_k}\nabla C(\theta_{k+1})
\end{aligned}
\end{equation}
Utilize the results from \eqref{eq: solved diff L theta_k in euclidean (appendix)} and \eqref{eq: solved diff U theta_k in euclidean (appendix)}, we have 
\begin{equation}
\begin{aligned}
    &\nabla\mathcal L^R(\theta_k)p^R\\
    &+\frac12H^L_k\left(H^R_*\right)^{-1}\left[\nabla \mathcal L^{L}(\theta_k;\boldsymbol{1}-\boldsymbol{\varepsilon}^{L}_{k}) 
    + \nabla \mathcal L^{U}(\theta_k;-\boldsymbol{\varepsilon}^{U}_{k})\right]p^L\\
    &+\frac12H^U_k\left(H^R_*\right)^{-1}\left[\nabla \mathcal L^{L}(\theta_k;\boldsymbol{1}-\boldsymbol{\varepsilon}^{L}_{k}) 
    + \nabla \mathcal L^{U}(\theta_k;-\boldsymbol{\varepsilon}^{U}_{k})\right]p^U\\
    &+\frac{1}{\alpha_k}(\theta_{k+1}-\theta_k)\\
    &=0.
\end{aligned}
\end{equation}
We thus conclude that
\begin{equation}\label{eq: solved steepest problem in euclidean (appendix)}
\begin{aligned}
    \Rightarrow \theta_{k+1}-\theta_k:\approx-\alpha_k
    [
    &\underbrace{\nabla \mathcal L^R(\theta_k)}_{(R)}p^R\\
    +\underbrace{\frac12(H^L_kp^L + H^U_kp^U)(H_*^R)^{-1}}_{(S)}[&\underbrace{\nabla \mathcal L^L(\theta_k;\boldsymbol{1}-\boldsymbol{\varepsilon}^L_k)}_{(L)}] \\
    + \underbrace{\frac12(H^L_kp^L + H^U_kp^U)(H_*^R)^{-1}}_{(S)}[&\underbrace{\nabla \mathcal L^U(\theta_k;-\boldsymbol{\varepsilon}^U_k)}_{(U)}]
    ].
\end{aligned}
\end{equation}
\end{proof}

\subsection{Proof of \textbf{Proposition.\,2} in Section.\,4.3}
\begin{prop}\label{prop: clu remain (appendix)}
    Using the model output KL divergence on the remaining set as the manifold metric, $\rho(\theta_k,\theta_{k+1})=D^R_{\text{KL}}\left(p_{z^R}(\theta_k)||p_{z^R}(\theta_{k+1}))\right)$. Assuming that the current model satisfies \textbf{Assumption.\,4}. Let $H_k^R=\nabla^2\mathcal L^R(\theta_k)$ represent the Hessian w.r.t. $\theta_k$ on the remaining set, then the steepest descent direction that minimizes \eqref{eq: optimization problem (appendix)} is approximately:
\begin{equation}\label{eq: clu remain (appendix)}
\begin{aligned}    
\Rightarrow\theta_{k+1}-\theta_k:\approx 
&-\frac{\alpha_k}{p^R+1}        
\underbrace{(H_k^R)^{-1}}_{(R)} \cdot\\          
&\underbrace{\frac12(H^L_kp^L + H^U_kp^U)(H_*^R)^{-1}}_{(S)}\cdot\\    
&[
    \underbrace{\nabla \mathcal L^L(\theta_k;\boldsymbol{1}-\boldsymbol{\varepsilon}^L_k)}_{(L)} + 
    \underbrace{\nabla \mathcal L^U(\theta_k;-\boldsymbol{\varepsilon}^U_k)}_{(U)}      
].
\end{aligned}
\end{equation}
\end{prop} 

\begin{proof}
Now, the steepest descent optimization problem for approximate CLU is as follows:
\begin{equation}
\begin{aligned}\label{eq: steepest problem in KL (appendix)}
&\theta_{k+1}\\
=&\mathop{\arg\min}\limits_{\theta_{k+1}} D_{\text{KL}}\left(p_z(\theta_*)||p_z(\theta_{k+1}))\right)\\
&+\frac{1}{\alpha_k}D_{\text{KL}}\left(p_{z^R}(\theta_k)||p_{z^R}(\theta_{k+1}))\right) \\
=&\mathop{\arg\min}\limits_{\theta_{k+1}}
\underbrace{D_{\text{KL}}\left(p_{z^R}(\theta_*)||p_{z^R}(\theta_{k+1}))\right)}_{(R(\theta_{k+1}))}p^{\mathcal Q}\\
&+ \underbrace{D_{\text{KL}}\left(p_{z^L}(\theta_*)||p_{z^L}(\theta_{k+1}))\right)}_{(L(\theta_{k+1}))}p^L \\
&+ \underbrace{D_{\text{KL}}\left(p_{z^U}(\theta_*)||p_{z^U}(\theta_{k+1}))\right)}_{(U(\theta_{k+1}))}p^U \\
&+ \frac{1}{\alpha_k}\underbrace{D_{\text{KL}}\left(p_{z^R}(\theta_k)||p_{z^R}(\theta_{k+1}))\right)}_{(C(\theta_{k+1}))}
\end{aligned} 
\end{equation}
The result of learning $L(\theta_{k+1})$ and unlearning part $U(\theta_{k+1})$ is the same as that in the Euclidean distance metric. And the remaining part $R(\theta_{k+1})$ and the constraint $C(\theta_{k+1})$ vary due to the output KL divergence metric $D^R_{\text{KL}}$. Note that $\nabla\mathcal L^R(\theta_{k+1})\approx\nabla\mathcal L^R(\theta_k)\approx\nabla\mathcal L^R(\theta_0)\approx0$. This enables us to take the second-order Taylor expansion at $\theta_k$ for the remaining part and the constraint.
\begin{equation}\label{eq: unsolved R theta_k+1 in KL (appendix)}
\begin{aligned}
    R(\theta_{k+1})
    &=R(\theta_k) + \nabla R(\theta_k)^\top(\theta_{k+1}-\theta_k)\\
    &+\frac12(\theta_{k+1}-\theta_k)^\top\nabla^2 R(\theta_k)(\theta_{k+1}-\theta_k) 
\end{aligned}
\end{equation}
\begin{equation}\label{eq: solved diff R theta_k in KL (appendix)}
    \nabla R(\theta_k)=- \mathbb{E}_{p(z^R;\theta_*)}\left[\nabla\log p(z^R;\theta_{k})\right] = \nabla \mathcal L^R(\theta_k)\approx 0
\end{equation}
\begin{equation}\label{eq: solved sec diff R theta_k in KL (appendix)}
    \nabla^2 R(\theta_k)=- \mathbb{E}_{p(z^R;\theta_*)}\left[\nabla^2\log p(z^R;\theta_{k})\right] = \nabla^2 \mathcal L^R(\theta_k)
\end{equation}
Substituting \eqref{eq: solved diff R theta_k in KL (appendix)} and \eqref{eq: solved sec diff R theta_k in KL (appendix)} into \eqref{eq: unsolved R theta_k+1 in KL (appendix)}, we can derive the remaining part:
\begin{equation}\label{eq: solved R theta_k+1 in KL (appendix)}
    R(\theta_{k+1})
    \approx R(\theta_k) + \frac12(\theta_{k+1}-\theta_k)^\top H^R_k (\theta_{k+1}-\theta_k)
\end{equation}
As for the constraint, we have 
\begin{equation}\label{eq: unsolved C theta_k+1 in KL (appendix)}
\begin{aligned}
    C(\theta_{k+1})
    &=C(\theta_k) + \nabla C(\theta_k)^\top(\theta_{k+1}-\theta_k)\\
    &+\frac12(\theta_{k+1}-\theta_k)^\top\nabla^2 C(\theta_k)(\theta_{k+1}-\theta_k) 
\end{aligned}
\end{equation}
\begin{equation}\label{eq: solved diff C theta_k in KL (appendix)}
    \nabla C(\theta_k)=- \mathbb{E}_{p(z^R;\theta_k)}\left[\nabla\log p(z^R;\theta_{k})\right] = 0
\end{equation}

\begin{equation}\label{eq: solved sec diff C theta_k in KL (appendix)}
\begin{aligned}
    \nabla^2 C(\theta_k)&=- \mathbb{E}_{p(z^R;\theta_k)}\left[\nabla^2\log p(z^R;\theta_{k})\right] \\
    &= F^R_k \approx H^R_k = \nabla^2 \mathcal L^R(\theta_k)
\end{aligned}
\end{equation}
Substituting \eqref{eq: solved diff C theta_k in KL (appendix)}and \eqref{eq: solved sec diff C theta_k in KL (appendix)} into \eqref{eq: unsolved C theta_k+1 in KL (appendix)}, we get the constraint
\begin{equation}\label{eq: solved C theta_k+1 in KL (appendix)}
    C(\theta_{k+1})
    \approx C(\theta_k) + \frac12(\theta_{k+1}-\theta_k)^\top H^R_k (\theta_{k+1}-\theta_k).
\end{equation}
Bringing \eqref{eq: solved diff L theta_k in euclidean (appendix)}, \eqref{eq: solved diff U theta_k in euclidean (appendix)}, \eqref{eq: solved R theta_k+1 in KL (appendix)}, and \eqref{eq: solved C theta_k+1 in KL (appendix)} into \eqref{eq: steepest problem in KL (appendix)}, and take the derivative w.r.t. $\theta_{k+1}$ of the minimization problem to derive the optimal solution, we have
\begin{equation}
\begin{aligned}
    0=&\nabla R(\theta_{k+1})p^R + \nabla L(\theta_{k+1})p^L \\
    &+ \nabla U(\theta_{k+1})p^U + \frac{1}{\alpha_k}\nabla C(\theta_{k+1})
\end{aligned}
\end{equation}
\begin{equation}
\begin{aligned}
    &\Rightarrow H^R_k (\theta_{k+1}-\theta_k)p^R\\
    &+\frac12H^L_k\left(H^R_*\right)^{-1}\left[\nabla \mathcal L^{L}(\theta_k;\boldsymbol{1}-\boldsymbol{\varepsilon}^{L}_{k}) 
    + \nabla \mathcal L^{U}(\theta_k;-\boldsymbol{\varepsilon}^{U}_{k})\right]p^L\\
    &+\frac12H^U_k\left(H^R_*\right)^{-1}\left[\nabla \mathcal L^{L}(\theta_k;\boldsymbol{1}-\boldsymbol{\varepsilon}^{L}_{k}) 
    + \nabla \mathcal L^{U}(\theta_k;-\boldsymbol{\varepsilon}^{U}_{k})\right]p^U\\
    &+\frac{1}{\alpha_k} H^R_k (\theta_{k+1}-\theta_k)\\
    &=0.
\end{aligned} 
\end{equation}
By rearranging the terms, we get
\begin{equation}\label{eq: solved steepest problem in KL (appendix)}
\begin{aligned}    
\Rightarrow\theta_{k+1}-\theta_k:\approx 
&-\frac{\alpha_k}{p^R+1}        
\underbrace{(H_k^R)^{-1}}_{(R)} \cdot\\          
&\underbrace{\frac12(H^L_kp^L + H^U_kp^U)(H_*^R)^{-1}}_{(S)}\cdot\\    
&[
    \underbrace{\nabla \mathcal L^L(\theta_k;\boldsymbol{1}-\boldsymbol{\varepsilon}^L_k)}_{(L)} + 
    \underbrace{\nabla \mathcal L^U(\theta_k;-\boldsymbol{\varepsilon}^U_k)}_{(U)}      
].
\end{aligned}
\end{equation}

\end{proof}

\subsection{Proof of \textbf{Proposition.\,3} in Section.\,5.2}
\begin{prop}\label{prop: implicit hessian approx (appendix)} 
For implicit online Hessian approximation in fast update, suppose $\beta_k,\delta_k$ is small, $\beta_k<\sqrt{\delta_k/|\nabla \mathcal L^R(\theta_k)-[\nabla \mathcal L^R(\theta_k)]^2|}$, $\mathcal L^R$ is $\mu$-smooth, i.e., $\lVert \nabla \mathcal L^R(\theta) - \nabla \mathcal L^R(\theta') \rVert_2\leq \mu\lVert\theta-\theta' \rVert_2$, and there exist an $\zeta_k$-neighborhood $\mathcal N(\theta_k^R,\zeta_k)$ of the optimal model parameter $\theta_k^R=\mathop{\arg\min}_{\theta^{\mathcal Q}_{k}}\mathcal L^R(\theta^{\mathcal Q}_k)$, which includes $\theta_k$ and $\theta_k^{\mathcal Q}$. Then, the iterative update term approximately is,
\begin{equation}\label{eq: implicit hessian update  (appendix)}
\begin{aligned}
    \theta_k-\theta_k^R :\approx &\beta_k^2 \left[ \nabla^2\mathcal L^R(\theta_k)\right]^{-1}\nabla \mathcal L^{\mathcal Q}(\theta_k)\\
    =&\beta_k^2 (H_k^R)^{-1}\nabla \mathcal L^{\mathcal Q}(\theta_k).
\end{aligned}
\end{equation}
\end{prop}

\begin{proof}
The objective function of implicit Hessian approximation can be formulated as:
\begin{equation}\label{eq: fast update (appendix)}
\mathop{\min}_{\theta^{\mathcal Q}_{k}}\mathcal L^R(\theta_k^{\mathcal Q})\quad \text{s.t.}\quad \theta_k^{\mathcal Q}=\theta_k-\beta_k\nabla \mathcal L^{\mathcal Q}(\theta_k),
\end{equation}
We need to get the optimal parameter $\theta_k^R$ that minimizes \eqref{eq: fast update (appendix)}, which means $0 = \frac{\partial\mathcal L^R(\theta_k^R)}{\partial \theta_k^R}$. We can take the Taylor expansion at $\theta_k$,
\begin{equation}
\begin{aligned}\label{eq: unsolved single iha (appendix)}
    0 &= \frac{\partial\mathcal L^R(\theta_k^R)}{\partial \theta_k^R} = \nabla \mathcal L^R(\theta_k^R) \\
    &= \nabla \mathcal L^R(\theta_k) + H_k^R(\theta_k^R-\theta_k) \\
    &+ (\theta_k^R-\theta_k)^\top \otimes \mathbf{T} \otimes (\theta_k^R-\theta_k) + o(\zeta_k)
\end{aligned}
\end{equation}
where $H_k^R=\nabla^2 \mathcal L^R(\theta_k)$ and $\mathbf T$ represent the Hessian matrix and the third-order symmetric tensor on the remaining set, respectively, and $\otimes$ denotes the Kronecker product.

From (\textbf{A2}) and (\textbf{A3}), we can reduce the first-order term to $o(\mu\zeta_k)$,
\begin{equation}\label{eq: first-order in iha (appendix)}
    \lVert \nabla \mathcal L^R(\theta^R_k) - \nabla \mathcal L^R(\theta_k) \rVert_2\leq \mu\lVert\theta^R_k-\theta_k \rVert_2 \leq \mu\zeta_k.
\end{equation}

To simplify the second-order term with (\textbf{A1}), we have 
\begin{align}\label{eq: second-order in iha (appendix)}
    &(\theta_k^R-\theta_k)^\top \otimes \mathbf{T} \otimes (\theta_k^R-\theta_k) \nonumber \\
    =& (\theta_k^{\mathcal Q}-\theta_k)^\top \otimes \mathbf{T} \otimes (\theta_k^{\mathcal Q}-\theta_k) + o(\epsilon_k) \nonumber \\
    =& \mathbf{C} \odot (\theta_k^{\mathcal Q}-\theta_k)^2 + o(\epsilon_k) \nonumber \\
    \approx & \beta^2\left(\nabla \mathcal L^{\mathcal Q}(\theta_k)\right)^2 + o(\epsilon_k) \nonumber \\
    =& \beta^2 \nabla \mathcal L^{\mathcal Q}(\theta_k) + o(\delta_k) + o(\epsilon_k) 
\end{align}

Bringing \eqref{eq: first-order in iha (appendix)} and \eqref{eq: second-order in iha (appendix)} into \eqref{eq: unsolved single iha (appendix)}, we have 
\begin{equation}
\begin{aligned}
    0 &= \nabla \mathcal L^R(\theta_k^R) \\
    &\approx H_k^R(\theta_k^R-\theta_k) + \beta^2 \nabla \mathcal L^{\mathcal Q}(\theta_k) \\
    &+ o(\delta_k) + o(\epsilon_k) + o(\mu\epsilon_k) + o(\epsilon_k^2)
\end{aligned}
\end{equation}
Then, we can derive
\begin{equation}
    \theta_k-\theta_k^R \approx \beta_k^2 (H_k^R)^{-1}\nabla \mathcal L^{\mathcal Q}(\theta_k).
\end{equation}

\end{proof}
\section{Algorithm}
We present the algorithm of our proposed UG-CLU in \textbf{Algorithm.\,\ref{alg: ug-clu}}.

\begin{algorithm}[htb]
    \caption{The Algorithm of Proposed UG-CLU}
    \label{alg: ug-clu}
\begin{algorithmic}[1]
    \STATE {\bfseries Input:} Sequential tasks $\{\mathcal T_t\}_{t=1}^T$, memory buffer $\mathcal B^R$, model $\theta$, outer loop learning rate $\alpha$, inner loop iteration number $K_{\text{in}}$, outer loop iteration number $K_{\text{out}}$, initial inner loop learning rate for learning $\beta^L$, unlearning $\beta^U$, and remaining $\beta^R$, temperature scalar $\lambda^{\mathcal Q}$, weight saliency mask threshold $\gamma$ 
    \FOR{$t=1$ {\bfseries to} $T$}
        \STATE {\bfseries Initialize:} $\mathcal T_t=(\mathcal D_t, \mathcal Q_t), \theta_0=\theta_t$.
        \IF{$\mathcal Q_t = U$}
            \STATE Remove all related samples from memory buffer
        \ENDIF    
        \STATE Sample task sample batch {\bfseries from} $\mathcal D^{\mathcal Q}$
        \STATE Compute $G^R_{0}=|\nabla \mathcal L^R(\theta_k)|$
        \FOR{$k=1$ {\bfseries to} $K_{\text{out}}$}
        \IF{$\mathcal Q_t = L$}
            \STATE Compute adaptive coefficients $\tilde{\boldsymbol \varepsilon}^L_{k-1}$
            \STATE Compute $G^{\mathcal Q}_{k-1}=|\nabla \mathcal L^L(\theta_{k-1},\boldsymbol{1}-\tilde{\boldsymbol{\varepsilon}}^L_{k-1}))|$
            \STATE Compute weight saliency mask by $\mathbf{m}=\mathbf{I}[G^{\mathcal Q}_{k-1}(G^R_{0})^{-1} \geq \gamma]$
            \STATE $\theta_{k-1}^{\mathcal Q}=\theta_{k-1}-\beta^L[\mathbf m\odot\nabla\mathcal L^L(\theta_{k-1};\boldsymbol{1}-\tilde{\boldsymbol{\varepsilon}}^L_{k-1})]$
        \ELSE
            \STATE Compute adaptive coefficients $\tilde{\boldsymbol \varepsilon}^U_{k-1}$
            \STATE Compute $G^{\mathcal Q}_{k-1}=|\nabla \mathcal L^U(\theta_{k-1},-\tilde{\boldsymbol{\varepsilon}}^U_{k-1}))|$
            \STATE Compute weight saliency mask by $\mathbf{m}=\mathbf{I}[G^{\mathcal Q}_{k-1}(G^R_{0})^{-1} \geq \gamma]$
            \STATE $\theta_{k-1}^{\mathcal Q}=\theta_{k-1}-\beta^U[\mathbf m\odot\nabla\mathcal L^U(\theta_{k-1};-\tilde{\boldsymbol{\varepsilon}}^U_{k-1})]$
        \ENDIF
            \STATE $\theta^R_0=\theta_{k-1}^{\mathcal Q}$
            \FOR{$k'=1$ {\bfseries to} $K_{\text{in}}$}
                \STATE Sample remaining sample batch {\bfseries from} memory buffer $\mathcal B^R$
                \STATE $\theta^R_{k'}=\theta^R_{k'-1}-\beta^R\nabla\mathcal L^R(\theta^R_{k'-1})$
        \ENDFOR
        \STATE $\theta_k=\theta_{k-1}-\alpha(\theta_{k-1}-\theta_{K_{\text{in}}}^{\mathcal Q})$
        \STATE Update memory buffer by task sample batch
    \ENDFOR

    \ENDFOR
\end{algorithmic}
\end{algorithm}
\section{Evaluation Metrics}

\begin{itemize}
    \item[$\bullet$] Learning accuracy (\textbf{LA}): We use LA to quantify the model's knowledge acquisition capability during CL process. This metric is calculated as the average classification accuracy of the final model on the test sets of all seen classes. Specifically, let $a_{t,j} \in [0,1]$ denote the classification accuracy on the classes appearing in the $j$-th task after the incremental learning of the $t$-th task $(j \leq t)$. The output space for calculating $a_{t,j}$ includes all previously seen classes. \textit{For CLU systems, a higher LA is desirable}. The LA metric can be calculated as follows:
    $$
        \mathrm{LA}=\frac{1}{T}\sum^T_{t=1}a_{T,t}
    $$
    \item[$\bullet$] Forget measure (\textbf{FM}): By measuring the degree of knowledge forgetting during the knowledge iteration process of the CLU system, we assess the stability of the agent. We evaluate forgetting by considering the difference $f_{T,j}$ between the final model's performance on all classes and its best past performance. \textit{For CLU systems, it is generally considered that, under similar LA conditions, a higher FM is better}. The specific calculation of FM is as follows:
    $$
        f_{T,j}=\max\limits_{i\in\{1,\dots,T\}}(a_{i,j}-a_{T,j})
    $$
    $$
        \mathrm{FM}=\frac{1}{T-1}\sum_{t=1}^{T-1}f_{T,t}
    $$
    \item[$\bullet$] Unlearning accuracy (\textbf{UA}): The effectiveness of the unlearning method is evaluated by assessing the model's average accuracy on all unlearning data $\mathcal D^U$. We consider the highest privacy leakage risk, meaning that we begin evaluating the average accuracy on the target unlearning data immediately after the unlearning request is initiated and the unlearning task is executed. This evaluation is repeated after each subsequent task until the sequence ends, and the highest accuracy for each class or sample is averaged to determine the UA. \textit{For CLU systems, a lower UA is preferable, indicating better privacy protection}.
    \item[$\bullet$] Membership inference attack (\textbf{MIA}): Following \cite{Foster2023FastMU}, we use a prediction-based membership inference attack to evaluate the privacy protection of knowledge in the CLU system. We first need to train an adversarial classifier to predict whether a particular example was contained in the training dataset. The predictions of the remaining dataset and the testing dataset by the unlearned model are collected. Based on these predictions, we calculate a prediction metric to train the attack classifier. Specifically, a Logistic Regression classifier is trained on the remaining prediction metrics labeled as ‘1’ and the testing prediction metrics labeled as ‘0’. This attack classifier is then applied to the unlearning prediction metrics to predict the membership of these unlearning samples. The success rate of the membership inference attack on $\mathcal D^U$ is quantified by the true positive rate predicted by our classifier, where $\mathrm{MIA} = TP/N_f$. Here, $TP$ represents the count of unlearning samples still identified as training samples, and $N_f$ is the size of the unlearning dataset. For the prediction metric, we considered prediction entropy, modified prediction entropy\cite{Foster2023FastMU}, and logit scaling from LiRA~\cite{carlini2022membership}, ultimately selecting the modified prediction entropy, which yielded the highest attack accuracy, as the prediction metric. Moreover, we consider the highest privacy leakage risk by measuring the MIA for unlearning data across all subsequent tasks and taking the average of the highest values as the final MIA. For CLU systems, a lower MIA is preferable, indicating better privacy protection.
    \item[$\bullet$] Clean accuracy (\textbf{CA}): For interclass confusion unlearning~\cite{Goel2022TowardsAE} in the task-agnostic CLU setting, we evaluate the final model’s average accuracy on the true labels of the confusion set. \textit{A higher value of this metric indicates better performance}.
    \item[$\bullet$] Output KL divergence with the oracle model (\textbf{KL}): We collect the predicted class probabilities of the final model and the oracle model on the complete dataset and then compute the KL divergence sample-wise. It is important to note that the KL divergence between the output distributions of oracle models with different random initializations is not necessarily zero due to the stochastic nature of the training algorithm. \textit{An ideal approximate CLU system aims to minimize the KL divergence between its output distribution and that of the optimal model}.
\end{itemize}
\section{Implementation Details}

For the CIFAR-10~\cite{Krizhevsky2009LearningML} dataset using ResNet-18~\cite{He2015DeepRL}, we employ the SGD optimizer with a weight decay of $5\times10^{-4}$, a batch size of $128$, 50 epochs for the learning task, and a default of 400 steps for the unlearning task. Our UG-CLU method uses a fixed outer loop learning rate of $\alpha=1.0$ and an inner loop iteration number of $K_{in}=1$. UG-CLU searches for the inner loop learning rate within the range $[0.1, 0.01]$ for learning, $[0.01, 0.001]$ for unlearning, and $[0.1, 0.01]$ for remaining tasks. The temperature scalar $\lambda^{\mathcal Q}$ is searched within $[0.0, 2.0]$, and the threshold $\gamma$ within $[0.5, 3.0]$. Experiments are conducted on 1 RTX 4090.

For the TinyImageNet~\cite{Le2015TinyIV} dataset, Swin-T~\cite{Liu2021SwinTH} is initialized from \texttt{torchvision} weights pre-trained on ImageNet. We use the AdamW optimizer~\cite{loshchilov2019decoupled} with a weight decay of $5\times10^{-2}$, a batch size of $64$, 10 epochs for the learning task, and a default of 150 steps for the unlearning task. UG-CLU searches for the inner loop learning rate within the range $[0.1, 0.01]$ for learning, $[0.01, 0.001]$ for unlearning, and $[0.1, 0.01]$ for remaining tasks. The temperature scalar $\lambda^{\mathcal Q}$ is searched within $[0.0, 2.0]$, and the threshold $\gamma$ within $[0.5, 3.0]$. Experiments are conducted on 1 RTX 4090.

\section{Additional Experimental Results}

\begin{figure*}
    \centering
    \includegraphics[width=\textwidth]{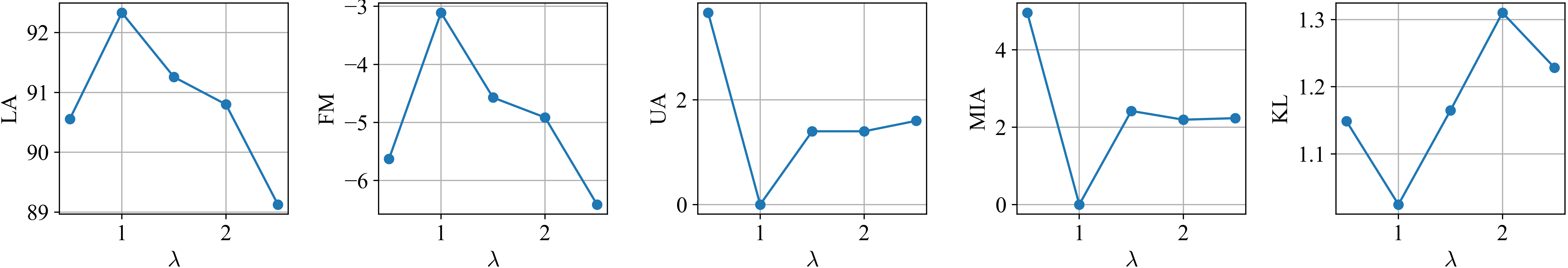}
    \caption{Sensitivity analysis of threshold $\lambda$ in weight saliency mask, evaluated on task-agnostic CLU setting, CIFAR-10 using ResNet-18.}
    \label{fig: metrics lambda}
\end{figure*}

\subsection{Sensitivity analysis on $\lambda$}
We first conducted a sensitivity analysis on the threshold $\lambda$ in the weight saliency mask. The trends of all metrics in task-agnostic CLU as $\lambda$ varies are shown in \textbf{Figure.\,\ref{fig: metrics lambda}}. When $\lambda$ is around $1$, the UG-CLU system achieves the best performance.
Notably, when $\lambda$ fluctuates, the relative amplitude of changes in the evaluation metrics is small, confirming the robustness of the algorithm to threshold selection.


\end{document}